\documentclass{article}

\usepackage{amsmath,amsfonts,bm}

\def\eqref#1{equation~\ref{#1}}

\def\1{\bm{1}}

\def\vtheta{{\bm{\theta}}}

\def\vtheta{{\bm{\theta}}}
\def\va{{\bm{a}}}
\def\vb{{\bm{b}}}

\def\vs{{\bm{s}}}
\def\vt{{\bm{t}}}

\def\vx{{\bm{x}}}

\def\m0{{\bm{0}}}
\def\mTheta{{\bm{\Theta}}}
\def\mA{{\bm{A}}}
\def\mB{{\bm{B}}}

\def\mH{{\bm{H}}}

\def\mO{{\bm{O}}}

\def\mU{{\bm{U}}}
\def\mV{{\bm{V}}}
\def\mW{{\bm{W}}}
\def\mX{{\bm{X}}}

\DeclareMathAlphabet{\mathsfit}{\encodingdefault}{\sfdefault}{m}{sl}
\SetMathAlphabet{\mathsfit}{bold}{\encodingdefault}{\sfdefault}{bx}{n}

\def\gE{{\mathcal{E}}}

\def\gG{{\mathcal{G}}}

\def\gL{{\mathcal{L}}}

\def\gN{{\mathcal{N}}}

\def\gR{{\mathcal{R}}}

\def\gV{{\mathcal{V}}}

\def\sR{{\mathbb{R}}}

\usepackage{multirow}
\usepackage{enumitem}
\usepackage{stfloats}

\usepackage{microtype}
\usepackage{graphicx}
\usepackage{subfigure}
\usepackage{booktabs}
\usepackage{hyperref}

\usepackage[accepted]{icml2024}

\usepackage{amsmath}
\usepackage{amssymb}
\usepackage{mathtools}
\usepackage{amsthm}

\usepackage[capitalize,noabbrev]{cleveref}

\theoremstyle{plain}
\newtheorem{theorem}{Theorem}[section]

\newtheorem{lemma}[theorem]{Lemma}

\theoremstyle{definition}

\theoremstyle{remark}

\usepackage[textsize=tiny]{todonotes}

\renewcommand{\eqref}[1]{(\ref{#1})}%

\usepackage{microtype}
\usepackage{graphicx}
\usepackage{subfigure}
\usepackage{booktabs} 

\usepackage{hyperref}

\makeatletter
\newcommand{\ovset}[3][0ex]{%
  \mathrel{\mathop{#3}\limits^{
    \vbox to#1{\kern-0.05\ex@
    \hbox{$\scriptstyle#2$}\vss}}}}
\makeatother

\icmltitlerunning{DUPLEX: Dual GAT for Complex Embedding of Directed Graphs}

\begin{document}

\twocolumn[
\icmltitle{DUPLEX: Dual GAT for Complex Embedding of Directed Graphs}

\icmlsetsymbol{equal}{*}

\begin{icmlauthorlist}
\icmlauthor{Zhaoru Ke}{equal,sht,ant}
\icmlauthor{Hang Yu}{equal,ant}
\icmlauthor{Jianguo Li}{ant}
\icmlauthor{Haipeng Zhang}{sht}
\end{icmlauthorlist}

\icmlaffiliation{sht}{School of Information Science and Technology, ShanghaiTech University, China}
\icmlaffiliation{ant}{Ant Group, China}

\icmlcorrespondingauthor{Haipeng Zhang}{zhanghp@shanghaitech.edu.cn}
\icmlcorrespondingauthor{Jianguo Li}{lijg.zero@antgroup.com}

\icmlkeywords{Directed graph, Graph embedding, Complex embedding, Graph neural network, Graph attention network}

\vskip 0.3in
]

\printAffiliationsAndNotice{\icmlEqualContribution} 

\begin{abstract}
Current directed graph embedding methods build upon undirected techniques but often inadequately capture directed edge information, leading to challenges such as: (1) \textbf{Suboptimal representations for nodes with low in/out-degrees}, due to the insufficient neighbor interactions; (2) \textbf{Limited inductive ability} for representing new nodes post-training; (3) \textbf{Narrow generalizability}, as training is overly coupled with specific tasks. In response, we propose DUPLEX, an inductive framework for complex embeddings of directed graphs. It (1) leverages Hermitian adjacency matrix decomposition for comprehensive neighbor integration, (2) employs a dual GAT encoder for directional neighbor modeling, and (3) features two parameter-free decoders to decouple training from particular tasks. DUPLEX outperforms state-of-the-art models, especially for nodes with sparse connectivity, and demonstrates robust inductive capability and adaptability across various tasks. The code is available at \url{https://github.com/alipay/DUPLEX}.

\end{abstract}

\section{Introduction}
\label{sec:intro}
Graphs, as a powerful and versatile data structure, have cemented their importance across a myriad of domains, including social science~\cite{hu2017deep}, recommendation systems~\cite{wu2022graph}, bioinformatics~\cite{yue2020graph}, traffic prediction~\cite{zheng2020gman}, financial and risk analysis~\cite{yang2021financial,yu2022efficient}. 
At the heart of graph analytics lies the concept of graph embedding, which seeks to encode complex, high-dimensional graph structures into compact, low-dimensional vector spaces. These representations are then applied to a variety of predictive tasks, such as link prediction and node classification. While the research on graph embeddings has achieved notable success, the primary focus has been on undirected graphs, which often fail to capture the intricate directional relationships inherent in many real-world networks.

\begin{figure}[t]
\begin{center}
\includegraphics[width=0.495
\textwidth]{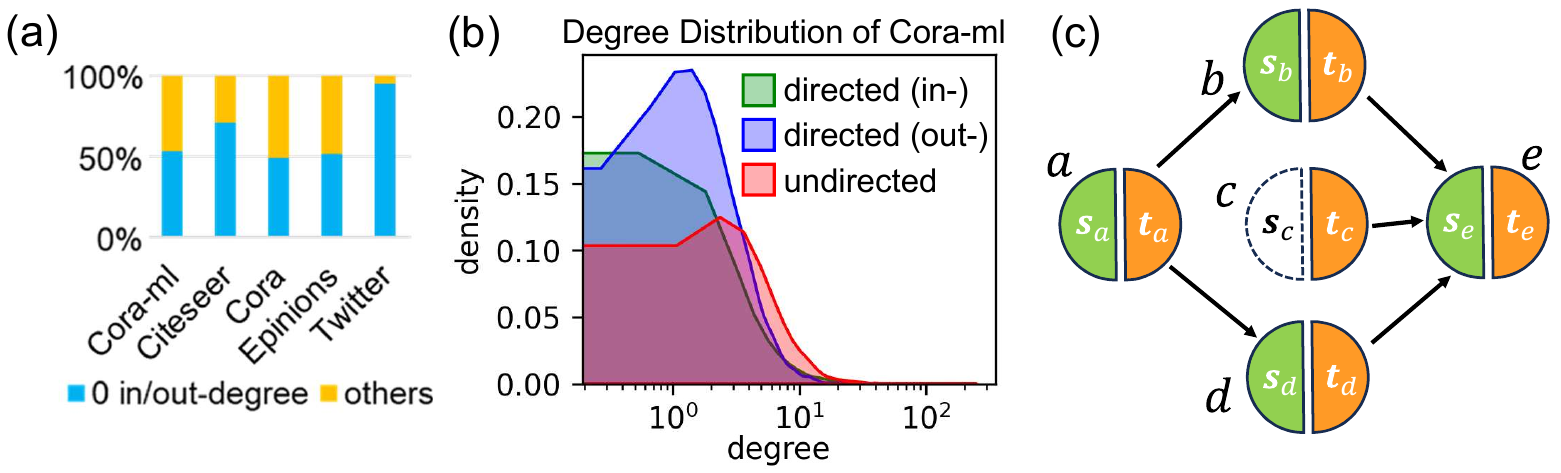}
\end{center}
\vspace{-2ex}
\caption{(a) Many nodes in real digraphs have zero in/out-degree.
(b) Separating in/out-degree yields more low-degree nodes compared to the total degree that disregards direction.
(c) Lack of in-neighbors hinders the dual embedding methods in capturing node $c$'s source role.}
\label{fig:intro}
\vspace{-4ex}
\end{figure}

Take, for example, the realm of social media, where users (nodes) follow one another, creating a digraph (directed graph) of influence and information flow. In this scenario, a directed edge from user $u$ to user $v$ implies that user $u$ follows user $v$, but not necessarily vice versa. The directionality of these edges encodes significant information about user behavior, influence patterns, and community structure. The need for digraph embedding (DGE) methods that can proficiently encode such asymmetrical structures is evident, yet the development of an effective DGE framework must also navigate several key desiderata to ensure practicality and utility:

\textbf{\textit{d1}.~Performance Across Various Node Degrees:}
Social networks are characterized by a wide disparity in user connectivity. An efficacious DGE method should consistently represent both influential personalities and average individuals, irrespective of the network's sparsity. This challenge is more pronounced in digraphs, where the separation of in and out neighbors often results in a higher proportion of nodes with a skewed degree distribution compared to their undirected counterparts (illustrated in Figure~\ref{fig:intro}(a)-(b)).

\textbf{\textit{d2}.~Transductive and Inductive Learning:}
The fluidity of social networks, marked by the continual arrival and departure of users, demands a DGE method capable of both updating embeddings for current users (transductive learning) and seamlessly extending to new users (inductive learning), without re-training from scratch.

\textbf{\textit{d3}.~Task Agnosticism:}
The embeddings produced by the DGE method should lend themselves to a variety of tasks such as detecting communities, recommending new connections, and identifying influential users, without being tailored to any single application.

Regrettably, our survey of the existing landscape reveals a gap: \textbf{existing DGE methods fall short in simultaneously satisfying all these desiderata}. As evidenced in Table~\ref{table:methods_compare}, while techniques that rely on self-supervised spatial GNNs (graph neural networks)~\cite{Ou2016AsymmetricTP, Zhou2017ScalableGE, Khosla2018NodeRL,Yoo2023DisentanglingDB, Kollias2022DirectedGA, Virinchi2023BLADEBN} demonstrate robustness against inductive bias and different downstream tasks, they typically learn dual embeddings (source and target roles) for each node, so as to better represent edge asymmetries. Consequently, they are prone to inferior performance for low in/out-degree nodes, since it fails to update the source embedding for low-out-degree nodes and the target embedding for low-in-degree nodes due to insufficient neighbors for training (cf. Figure~\ref{fig:intro}(c)). On the flip side, single-embedding strategies~\cite{Zhang2021MagNetAN, Lin2022AMF, Fiorini2022SigMaNetOL, Ko2023UniversalGC} can alleviate this issue by aggregating information from both in and out neighbors. However, they are generally bound to spectral GNN architectures and fully supervised paradigms, limiting their scope to specific tasks within a transductive context.

\begin{table}[t]
\tabcolsep = 0.06cm
\begin{center}
\caption{Existing DGE methods and their specifics.}
\vspace{-1ex}
\label{table:methods_compare}
\resizebox{\linewidth}{!}{
\begin{tabular}{lccc}
\toprule
 & Embedding & Network & Learning Algo. \\ \midrule
\begin{tabular}[c]{@{}l@{}}HOPE~\cite{Ou2016AsymmetricTP},\\ APP~\cite{Zhou2017ScalableGE},\\ NERD~\cite{Khosla2018NodeRL},\\ ODIN~\cite{Yoo2023DisentanglingDB},\\ DGGAN~\cite{Zhu2020AdversarialDG}\end{tabular} & dual & \begin{tabular}[c]{@{}l@{}}shallow \\ (no network)\end{tabular} & self-supervised \\ \hline
DiGAE~\cite{Kollias2022DirectedGA} & dual & spectral GNN & self-supervised \\ \hline
\begin{tabular}[c]{@{}l@{}}BLADE~\cite{Virinchi2023BLADEBN},\\ CoBA~\cite{liu2023collaborative}\end{tabular} & dual & spatial GNN & self-supervised \\ \hline
Gravity GAE~\cite{SalhaGalvan2019GravityInspiredGA} & dual & spectral GNN & self-supervised \\ \hline
Dhypr~\cite{Zhou2021DHYPRHN} & dual & spatial GNN & self-supervised \\ \hline
\begin{tabular}[c]{@{}l@{}}MagNet~\cite{Zhang2021MagNetAN},\\ Framelet-MagNet~\cite{Lin2022AMF},\\ SigMaNet~\cite{Fiorini2022SigMaNetOL}\end{tabular} & single & spectral GNN & supervised \\ \hline
UGCL~\cite{Ko2023UniversalGC} & single & spectral GNN & self-supervised \\ \hline
DUPLEX (proposed method) & single & spatial GNN & self-supervised \\
\bottomrule
\end{tabular}
}
\vspace{-4ex}
\end{center}
\end{table}
To bridge the above gap, we propose \textbf{DUPLEX} (\textbf{DU}al graph attention networks for com\textbf{PLEX} embedding of digraphs), a novel approach for digraph embedding. Our method leverages a dual graph attention network (GAT) encoder and two parameter-free decoders to learn a single complex embedding for each node. Specifically, to tackle the problem for those low-degree nodes (\textbf{\emph{d1}}), DUPLEX embraces complex embeddings, underpinning the source and target roles as complex conjugates, which facilitates a collaborative optimization leveraging both in and out neighbors. For the inductive power (\textbf{\emph{d2}}), DUPLEX employs a dual GAT encoder with separate components for the amplitude and phase parts of the complex node embedding. The amplitude encoder captures connection information with an undirected graph aggregator, while the phase encoder characterizes direction information using a digraph aggregator. Both aggregators update node embeddings by collecting information only from neighboring nodes, eliminating the need to access the entire graph and generalizing well to unseen nodes.  Lastly, to achieve task generalization (\textbf{\emph{d3}}), DUPLEX employs the two parameter-free decoders (direction-aware and connection-aware) for reconstructing the Hermitian adjacency matrix (HAM) of the digraph. The node embeddings are then trained in a self-supervised manner, preserving the structural characteristics. Thus, the resulting node embedding can adapt effectively to various downstream tasks. Our contributions can be summarized as:
\begin{itemize}[leftmargin=1em, itemsep= -2pt, topsep = 0pt, partopsep=0pt]
 \item We propose DUPLEX that learns \textbf{complex} node embeddings. Our approach incorporates a \textbf{dual GAT encoder}, comprising two specially designed graph aggregators for the amplitude and phase components. To our knowledge, this is the first exploration of using spatial GNNs for complex embeddings of digraphs.
 \item We propose two parameter-free decoders specially designed for complex embeddings, targeting at the reconstruction of the HAM of the digraph. The model can be trained in a self-supervised manner, enabling the adaptability of node embeddings across various tasks.
 \item We conduct comprehensive experiments on diverse tasks and datasets, showcasing the superior performance of our approach. The results demonstrate its efficacy in modeling low-degree nodes, generalizing to multiple tasks, and handling unseen nodes.
\end{itemize}

\section{Related Works}
The following discussion reviews existing digraph embedding (DGE) methodologies from three perspectives: the nature of the embeddings (dual vs. single), the network designs (transductive vs. inductive), and the learning paradigms (supervised vs. self-supervised). These insights further underscore the innovative design of DUPLEX.

\subsection{Dual versus Single Node Embeddings}
\label{sec:two-emb}
As shown in the second column of Table~\ref{table:methods_compare}, studies on DGE have largely diverged into two camps based on how they treat node roles: those that generate dual embeddings to capture distinct source and target roles and those that distill this information into a single embedding. Pioneering approaches like APP~\cite{Zhou2017ScalableGE} often fall short for sparsely connected nodes due to inadequate sampling of these nodes within random walks. Although improved sampling strategies from NERD~\cite{Khosla2018NodeRL} and BLADE~\cite{Virinchi2023BLADEBN} aim to remedy this, their effectiveness is curtailed by the small neighborhood sets of low-degree nodes, which are insufficient for accurate embeddings. DGGAN~\cite{Zhu2020AdversarialDG} and CoBA~\cite{liu2023collaborative} strive to settle this problem by bolstering the link between source and target embeddings, but only in a heuristic way. Alternatively, GravityGAE~\cite{SalhaGalvan2019GravityInspiredGA} and its hyperbolic extension~\cite{Zhou2021DHYPRHN} employs a gravitational analogy with a mass parameter only for the target role, yet struggles with nodes of meager in-degrees.

Conversely, single-embedding methods, exemplified by MagNet~\cite{Zhang2021MagNetAN} and its extensions~\cite{Lin2022AMF,Fiorini2022SigMaNetOL,Ko2023UniversalGC}, effectively consolidate information from both in and out neighbors into a unified complex embedding, overcoming the embedding quality challenges for nodes with low in/out-degrees.

DUPLEX aligns with the latter approach, using the HAM akin to the magnetic Laplacian. Yet, it advances the notion by modeling the source and target roles as complex conjugates, enabling joint optimization while preserving edge asymmetries.

\subsection{Transductive versus Inductive Network Designs}

As shown in the third column of Table~\ref{table:methods_compare}, the body of DGE work can also be classified based on its adaptability to unseen data. Shallow methods such as HOPE~\cite{Ou2016AsymmetricTP} and APP~\cite{Zhou2017ScalableGE}, which adapt classical techniques like matrix factorization and random walks to digraphs, operate under transductive settings where embeddings are directly learned as trainable parameters. Spectral GNNs, including DiGAE~\cite{Kollias2022DirectedGA}, Gravity GAE~\cite{SalhaGalvan2019GravityInspiredGA}, and MagNet~\cite{Zhang2021MagNetAN}, further necessitate full graph knowledge during graph Fourier transform, limiting their applicability to new nodes. Spatial GNNs, represented by BLADE~\cite{Virinchi2023BLADEBN}, CoBA~\cite{liu2023collaborative}, and Dhypr~\cite{Zhou2021DHYPRHN}, however, infer embeddings by aggregating local neighborhood information, enabling them to generalize to nodes absent during training.

DUPLEX builds upon this spatial approach with a dual encoder architecture that separately tackles the amplitude and phase aspects of embeddings, thereby accommodating both transductive and inductive learning effectively.

\subsection{Supervised versus Self-supervised Learning}

As shown in the fourth column of Table~\ref{table:methods_compare}, supervised methods such as MagNet~\cite{Ou2016AsymmetricTP} and subsequent variations~\cite{Lin2022AMF,Fiorini2022SigMaNetOL,Ko2023UniversalGC} leverage spectral GNN encoders with linear or convolutional decoders, trained end-to-end for specific tasks. This task-specific focus can hinder the generalizability of node embeddings, as they may not fully capture the graph's structural features. On the other hand, self-supervised methods~\cite{Ou2016AsymmetricTP,Yoo2023DisentanglingDB,liu2023collaborative} differ by encoding both connectivity and directionality within the graph, producing embeddings that are suitable for various downstream tasks.

As a result, DUPLEX employs a self-supervised method and reconstructs the Hermitian adjacency matrix through two parameter-free decoders, such that the embeddings given by the learnable encoder not only align with the graph's structural properties but also ensure the adaptability of the embeddings across multiple tasks.

\begin{figure*}[t]
\vspace{-0ex}
\begin{center}
\includegraphics[width=0.99\textwidth]{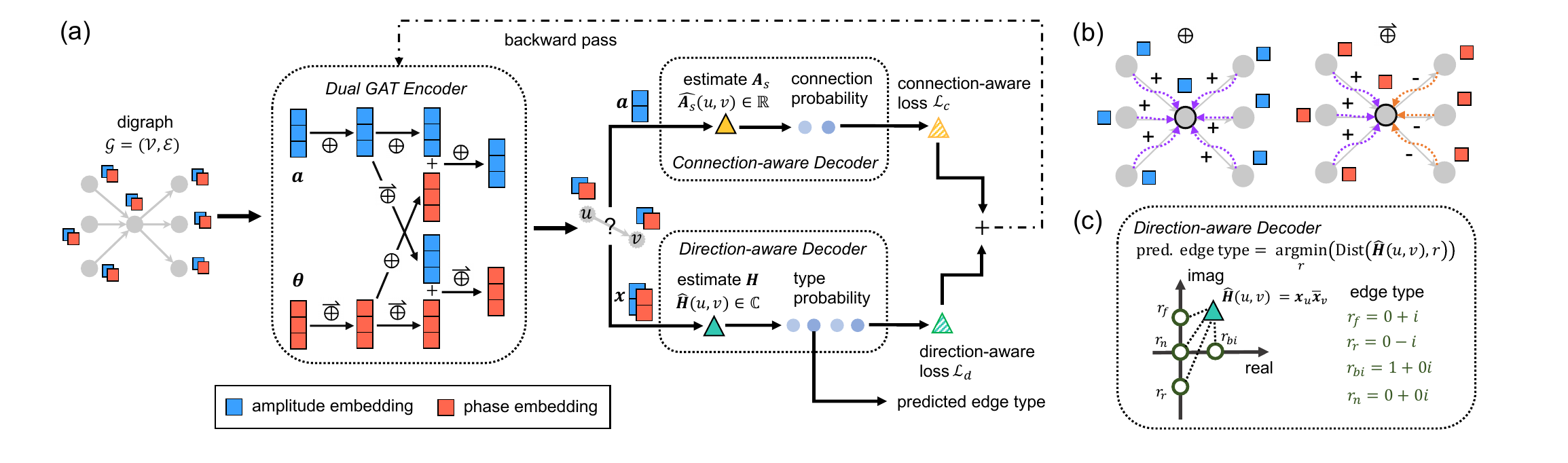}
\end{center}
\vspace{-1ex}
\caption{The architecture of DUPLEX. (a) Forward pass and backward pass of the model. (b) The undirected ($\oplus$) and directed ($\ovset{\rightharpoonup}{\oplus}$) graph aggregator. (c) The main idea of the direction-aware decoder.}
\label{fig:framework}
\vspace{-1ex}
\end{figure*}

\section{Method}
\label{sec:method}
A digraph (directed graph) $\gG = \{\gV,\gE\}$ is defined by its nodes $\gV$ and directed edges $\gE$, with each edge $(u, v) \in \gE$ representing a connection from node $u$ to node $v$. Our goal is to represent each node $u$ by a $d$-dimensional complex vector $\vx_u \in \mathbb{C}^{d \times 1}$ that meets the desiderata presented in Section~\ref{sec:intro}. To this end, we introduce DUPLEX, a novel framework that is illustrated in Figure~\ref{fig:framework}(a). At the heart of DUPLEX lies the Hermitian adjacency matrix (HAM), whose symmetrical structure and complex-valued representation enable the integration of both directionality and connectivity within a cohesive framework. This mathematical formalism guides the design of DUPLEX's embeddings, encoder, decoder, and loss functions, ensuring coherence with HAM's fundamental principles. Specifically, beginning with initial embeddings\footnote{For non-attributed graphs, we use the standard normal distribution for initialization.}, the dual GAT encoder generates amplitude and phase embeddings that encode the graph's connectivity and directionality, respectively. A fusion layer within the encoder integrates these embeddings, while subsequent direction-aware and connection-aware decoders—both parameter-free—rebuild the HAM. This architecture ensures that the learned complex embeddings authentically represent the graph's structure, guided by self-supervised loss functions.

In the sequel, we introduce each component in DUPLEX, including the HAM and the resulting embedding formulations, the dual GAT encoder, the decoders, and the loss functions.

\subsection{Hermitian Adjacency Matrix Construction}
\label{sec:hermitian-at-cons}

Recall that classical undirected graph embedding methods often employ matrix factorization techniques that can be represented as $\mA = \mX^\intercal\mX$, where $\mA$ is the symmetric adjacency matrix and $\mX$ is the corresponding real-valued node embedding matrix. In contrast, digraphs feature asymmetric adjacency matrices, resulting in factorizations of the form $\mA = \mX_s^\intercal\mX_t$, with $\mX_s$ and $\mX_t$ denoting the source and target embeddings, respectively. As highlighted in Section~\ref{sec:two-emb}, such dual embedding strategies encounter difficulties in generating high-quality embeddings for nodes with low in/out-degree. More discussions from the perspective of SVD (singular value decomposition) are provided in Appendix~\ref{app:SVD_drawback}.

To address these limitations, we adopt the Hermitian adjacency matrix (HAM), a symmetric matrix utilized within spectral graph theory~\cite{Guo2015HermitianAM,liu2015hermitian} to represent digraphs. The HAM $\mH$ for a digraph $\gG$ is defined in polar form as:
\begin{equation}
 \label{eq:hermitian}
 \mH = \mA_s \odot \exp\left( i\frac{\pi}{2}\mTheta \right),
\end{equation}
where $i$ represents the imaginary unit, $\pi$ is the known mathematical constant, and $\odot$ signifies the Hadamard product. The symmetric binary matrix $\mA_s$ satisfies $\mA_s(u,v) = 1$ if $(u,v) \in \gE \vee (v,u) \in \gE$, and 0 otherwise. The antisymmetric matrix $\mTheta$ contains elements from the set $\{-1,0,1\}$, as defined by:
\begin{equation}
 \label{eq:theta}
 \mTheta(u,v)= 
 \begin{cases}
 1, & \text{if } (u,v) \in \gE, \\
 -1, & \text{if } (v,u) \in \gE,\\
 0, & \text{otherwise}.\\
 \end{cases} 
\end{equation}
As a result, $\mH(u,v)\in \{i, -i, 1, 0\}$ alone can represent forward, reverse, bidirectional, and no edge between $u$ and $v$, whereas the asymmetric adjacency matrix $\mA$ requires both $\mA(u,v)$ and $\mA(v,u)\in\{0,1\}$ for the same goal.

Moreover, the matrix decomposition $\mH = \mX^\intercal\bar{\mX}$ allows for the derivation of the node embedding $\vx_u$ in polar form:
\begin{align}
    \vx_u &= \va_u \odot \exp\left( i\frac{\pi}{2}\vtheta_u \right), \label{eq:node_emb_s} \\
    \bar{\vx}_u &= \va_u \odot \exp\left( -i\frac{\pi}{2}\vtheta_u \right). \label{eq:node_emb_t}
\end{align}
Here, $\va_u$ encapsulates the amplitude and $\vtheta_u$ the phase of $\vx_u$. We can interpret $\vx_u$ and $\bar{\vx}_u$ as the complex conjugate embeddings representing the source and target roles of node $u$, in contrast to the independent dual embeddings of previous methods. This embedding representation is a joint function of $\va_u$ and $\vtheta_u$, facilitating co-optimized learning from both in and out-neighborhoods, which will be explained in Section~\ref{sec:dual_enc}, thus resolving the issue of suboptimal embeddings for nodes with low degree connectivity. From another perspective, $\vx_u$ can be considered as a unified embedding vector for node $u$, with the HAM reconstructible through the Hermitian inner product between node embeddings instead of the ordinary inner product.

\subsection{Dual GAT Encoder}
\label{sec:dual_enc}
The encoder in DUPLEX aims to generate the amplitude $\va_u$ and phase $\vtheta_u$ embeddings, which combine to form the complex node embeddings $\vx_u$. One of the nice properties of the HAM is the separation of the amplitude $\mA_s$ and phase $\mTheta$ components to capture the connection and direction information, as shown in Eq.~(\ref{eq:hermitian}). Moreover, Eqs.~(\ref{eq:node_emb_s}) and (\ref{eq:node_emb_t}) reveal that the amplitude parts of the source and target embeddings are identical for a node, but the phase parts exhibit opposite signs. In other words, the phase embedding is direction-aware, whereas the amplitude embedding is not. This insight lays the groundwork for a dual encoder that separately processes amplitude and phase, enabling streamlined learning and halving the time complexity and parameter count compared to direct complex embeddings encoding (detailed in Appendix~\ref{app:complexity}).

Our dual encoder comprises an amplitude encoder, a phase encoder, and a fusion layer, all based on the spatial GAT backbone for its flexibility and scalability~\cite{Velickovic2017GraphAN}. Note that GAT can be replaced by other spatial GNNs in DUPLEX (see Section~\ref{sec:ablation}). Both encoders update embeddings by aggregating information from in and out neighbors, mitigating the issue of poor embeddings for low-degree nodes, as we further discuss now.

\subsubsection{Amplitude Encoder}
Since the amplitude embedding $\va_u$ captures connection information akin to undirected graphs, we utilize the original GAT~\cite{Velickovic2017GraphAN} to update $\va_u$:
\begin{equation}
 \label{eq:am_update}
 \va_u' = \phi\Big(\oplus\big(\{\va_v,\forall v\in\gN(u)\}\big)\Big),
\end{equation}
where $\va_u'$ denotes the updated amplitude embedding, $\gN(u)$ includes node $u$ and its neighbors, $\phi$ is the activation function (ReLU in our case), and $\oplus(\cdot)$ is the aggregator used for the undirected graph (see Figure~\ref{fig:framework}(b)):
\begin{equation}
 \label{eq:am_agg}
  \oplus\big(\{\va_v,\forall v\in\gN(u)\}\big) = \sum_{v\in\gN(u)}f_a(\va_u,\va_v)\psi_a(\va_v).
\end{equation}
Here, $f_a(\va_u,\va_v)\psi_a(\va_v)$ is a learnable attention mechanism that computes the relevance of neighbor embeddings to node $u$.

\subsubsection{Phase Encoder}  
Bridging the gap between undirected graph embedding approaches and the directionality of digraphs, the phase encoder introduces a novel, direction-aware graph aggregator. This aggregator respects the sign difference of $\vtheta_u$ in Eqs.~(\ref{eq:node_emb_s}) and (\ref{eq:node_emb_t}). Specifically, it consumes information from the in-neighbors and out-neighbors with different signs (as depicted in Figure~\ref{fig:framework}(b)), enabling the phase encoder to effectively incorporate the contributions from both types of neighbors while considering their inherent asymmetry:
\begin{align}
    \label{eq:ph_agg}
     &\ovset{\rightharpoonup}{\oplus}\big(\{\vtheta_v : v\in\gN(u)\}\big) 
    =\, \sum_{v\in\gN_{\text{in}}(u)}f_\theta(\vtheta_u,\vtheta_v)\psi_\theta(\vtheta_v) \notag \\ 
    &\qquad -\sum_{v\in\gN_{\text{out}}(u)}f_\theta(\vtheta_u,\vtheta_v)\psi_\theta(\vtheta_v), 
\end{align}
where $\gN_{\text{in}}(u)$ and $\gN_{\text{out}}(u)$ denote the sets of in-neighbors and out-neighbors, respectively. The resulting phase embedding in each layer can be updated as:
\begin{equation}
    \label{eq:ph_update}
        \vtheta_u' =  
        \phi\Big(\ovset{\rightharpoonup}{\oplus}\big(\{\vtheta_v : v\in\gN(u)\}\big)\Big).
\end{equation}
This design choice results in a phase encoder that is distinctly directional, setting it apart as a novel contribution to graph representation learning.

\subsubsection{Fusion Layer} 
\label{ssec:fusion_layer}
The fusion layer functions as a pivotal junction where the amplitude and phase embeddings intersect, each carrying unique yet complementary information essential for the reconstruction of the HAM. Let us consider the amplitude encoder as an illustrative example. In the fusion layer, we not only aggregate the information from the amplitude embeddings but also gather side information from the phase embeddings of the previous layer using the undirected graph aggregator $\oplus$. Mathematically, this can be formulated as:
\begin{equation}
 \label{eq:am_fu_update}
 \va_u' = \phi\Big(\oplus\big(\{\va_v,\forall v\in\gN(u)\}\big)+\oplus\big(\{\vtheta_v,\forall v\in\gN(u)\}\big)\Big).
\end{equation}
The resulting $\va_u'$ is then propagated to the subsequent layers in the amplitude encoder. On the other hand, for the phase encoder, the fusion layer can be similarly expressed by replacing $\oplus$ with $\ovset{\rightharpoonup}{\oplus}$ in Eq.~\eqref{eq:am_fu_update}.

Our approach employs a ``mid-fusion'' strategy, integrating embeddings at the network's intermediate layers for two main reasons. First, in the absence of attributes, node features are initially random, and early fusion might introduce noise rather than beneficial information. Second, fusing at the terminal layer could dilute the unique attributes of amplitude and phase embeddings that encode connection and directional information, respectively, potentially hindering the decoders' ability to accurately reconstruct the HAM. Consequently, the fusion layer is optimally placed as an intermediary, allowing for the coordinated optimization within the dual encoder framework.


The integration of the fusion layer within the DUPLEX framework is also founded on the hypothesis that an adequately flexible encoder can achieve superior performance by permitting the controlled interchange of information between the amplitude and phase. This fusion is implemented utilizing the attention mechanism, where the resulting attention score dictates the strength of the interchange. A critical aspect of this implementation is that when the attention score is minimal, nearing zero, it effectively precludes the exchange of information between the amplitude and phase. Consequently, the data itself, from which the attention scores are adaptively learned, governs the presence and extent of information exchange.

\textbf{Relation to GAT}: DUPLEX amounts to GAT for undirected graphs after removing the phase or direction-related components. In other words, DUPLEX is an extension of GAT that is augmented for digraphs through the integration of direction information, and so inherits the merits of GAT.

\subsection{Parameter-free Decoder and Self-supervised Loss}
\label{sec:decoder}
After obtaining the complex embeddings for nodes, DUPLEX proceeds to reconstruct the HAM with two parameter-free decoders (i.e., direction and connection-aware decoders), each paired with its own self-supervised loss function. Unlike previous methods in Table~\ref{table:methods_compare} that reconstruct the real-valued asymmetric adjacency matrices with the connection-aware decoder, the direction-aware decoder in DUPLEX is specially designed for the complex-valued HAM. Meanwhile, the connection-aware decoder complements the former by focusing on the binary existence of connections, effectively serving as an auxiliary to the direction-aware objective. The parameter-free decoders and self-supervised losses enable DUPLEX to learn both connectivity and directionality within the digraph, fostering node embeddings that generalize effectively across tasks without depending on tailored parameterized decoders.

\subsubsection{Direction-aware Decoder and Loss} 

Recall that HAM comprises four distinct elements, namely $\mH(u,v)\in \gR =\{i, -i, 1, 0\}$, which respectively signify forward, reverse, bidirectional, and no edges between any pair of nodes $u$ and $v$. The reconstruction task models the estimated matrix elements $\hat\mH(u,v) = \vx_u^\intercal\bar{\vx}_v$ to align with the ground truth $\mH(u,v)$. Recognizing the limitations of low-rank embeddings in capturing the full spectrum of the HAM, we approach the problem as a classification task. Each edge $(u,v)$ is assigned probabilities corresponding to its edge type, based on the relative distance between $\hat\mH(u,v)$ and $\mH(u,v)$ (illustrated in Figure~\ref{fig:framework}(c)):
\begin{equation}
 \label{eq:mul-decoder}
 P(\hat\mH(u,v) = r) = \frac{\exp(-|\vx_u^\intercal\bar{\vx}_v-r|)}{\sum_{r'\in \gR}\exp(-|\vx_u^\intercal\bar{\vx}_v-r'|)},   \quad \forall r \in \gR,
\end{equation}
where the L1 distance is used, due to its advantage over L2 (Appendix~\ref{app:distance}). We then minimize the negative log-likelihood of the samples pertaining to different edge types, resulting in the self-supervised direction-aware loss: 
\begin{align}
    \gL_d = -\sum_{r\in\gR}\sum_{\mH(u,v) = r} \log P(\hat\mH(u,v) = r).
\end{align}
 
\subsubsection{Connection-aware Decoder and Loss} 
The connection-aware decoder isolates the task of discerning node connectivity by reconstructing $\mA_s$, the amplitude component of the HAM~\eqref{eq:hermitian}, from the amplitude embeddings $\va_u$. It posits the connection probability for an edge $(u,v)$ as:
\begin{equation}
 \label{eq:ex-decoder}
 P(\hat\mA_s(u,v) = 1) = \sigma(\va_u^\intercal \va_v),
\end{equation}
where $\sigma$ is the sigmoid function and $\hat\mA_s$ is the estimated connection matrix. This decoder's loss function $\gL_c$ adheres to the same negative log-likelihood minimization principle as the direction-aware loss.

\subsubsection{Total Loss} 
\label{sec:ttloss}
The total loss can be written as: $\gL =\gL_d+\lambda\gL_c$, where $\lambda$ is the weight parameter for the connection-aware loss $\gL_c$. The optimization target of the connection-aware loss is subsumed within the broader objective of the direction-aware loss. In tandem, the connection-aware loss can support the direction-aware loss by constraining the optimization space in the initial stages. As a result, we start with a predefined value of $\lambda = \lambda_0$, and allow $\lambda$ to decay (i.e., $\lambda = \lambda_0 q^k$) with a decay factor $q < 1$ as the number of epochs $k$ increases, reflecting the decreasing necessity of the connection-aware loss as direction-aware accuracy improves. This objective function implicitly mandates $\va_u$ and $\vtheta_u$ to encapsulate the amplitude and phase information respectively, given the encoders possess sufficient flexibility to facilitate such characterization.

\subsubsection{Supervised Training of DUPLEX} 
Apart from self-supervised training, DUPLEX can be readily trained end-to-end in a supervised manner for specific downstream tasks. This entails replacing the decoder and loss function with a task-specific objective. Taking node classification as an example, given the complex embeddings encoded by the dual GAT encoder, we can simply concatenate the amplitude and phase embeddings. These concatenated embeddings are then mapped to node labels using a linear layer as in~\cite{Zhang2021MagNetAN}. The model is then trained using a cross-entropy loss function, optimizing its performance specifically for the node classification task.

\textbf{Relation to the MagNet series}: MagNet~\cite{Zhang2021MagNetAN} and its variants~\cite{Fiorini2022SigMaNetOL,Lin2022AMF} mainly extend spectral GCNs from undirected to digraphs by substituting the traditional Laplacian in graph convolutions with its magnetic counterpart, combining amplitude and phase embeddings, and utilizing a linear decoder for link prediction tasks. While both the magnet Laplacian and the HAM are Hermitian matrices, MagNet's approach is rooted in spectral domain convolutions. In contrast, DUPLEX innovatively employs spatial GNNs derived from the properties of HAM, with the goal of further reconstructing the HAM. Furthermore, unlike the MagNet series, which predominantly operates under fully supervised training, DUPLEX takes advantage of a self-supervised learning paradigm.

\section{Experiments}
In this section, we benchmark DUPELX against existing digraph embedding methods on three tasks, namely, link prediction, transductive node classification, and inductive node classification. Furthermore, we conduct a comprehensive ablation study on DUPLEX to assess the significance of its different components.

\subsection{Datasets and Experiment Set-up}

\begin{table}[t]
\caption{Dataset characteristics.}
\vspace{2ex}
\label{table:dataset}
\tabcolsep = 0.1cm
\begin{center}
\resizebox{\linewidth}{!}{
\begin{tabular}{lccccc} 
\toprule
Dataset  & $|\gV|$ & $|\gE|$ & \%Directed Edges & Feature Dim & \#Classes \\ \midrule
Cora-ml  & 2,995 & 8,416 & 93.9 & 2,879 & 7 \\
Citeseer & 3,312   & 4,715   & 95.0 & 3,703 & 6  \\
Cora & 23,166  & 91,500  & 94.9 & - & - \\
Epinions & 75,879  & 508,837 & 59.5 & - & - \\
Twitter  & 465,017 & 834,797 & 99.7 & - & -   \\ 
\bottomrule
\end{tabular}}
\end{center}
\vspace{-6ex}
\end{table}

The experiments are conducted on five public datasets of digraphs, namely Cora-ml, Citeseer, Cora, Epinions, and Twitter (overview in Table~\ref{table:dataset} and more details in Appendix~\ref{app:datasets}). Note that only Cora-ml and Citeseer datasets are suitable for node classification tasks, as they include node labels and initial attributes. The graph datasets used in our study are unweighted and directed. We then compare DUPLEX with a selection of state-of-the-art (SOTA) node embedding models for digraphs, including both dual and single embedding methods. The first group consists of \textbf{NERD}~\cite{Khosla2018NodeRL}, \textbf{DGGAN}~\cite{Zhu2020AdversarialDG}, \textbf{DiGAE}~\cite{Kollias2022DirectedGA} and \textbf{ODIN}~\cite{Yoo2023DisentanglingDB}. For the second group, we consider \textbf{MagNet}~\cite{Zhang2021MagNetAN} and \textbf{SigMaNet}~\cite{Fiorini2022SigMaNetOL}. In addition, we examine the performance of a simplified version of our model, denoted as \textbf{DUPLEX*}, which excludes the fusion layer. To ensure optimal performance, we perform hyperparameter selection for each method. Moreover, we repeat the experiments 10 times and report the average values, along with the standard deviation across the runs (surrounded by brackets). For implementation details, please see Appendix~\ref{app:implementation}.

\subsection{Downstream Tasks}
\subsubsection{Link Prediction}
\label{sec:lp}

\begin{table*}[t]
\vspace{-0ex}
\caption{Link prediction AUC (\%) and ACC (\%) for subtask 1 and 2. The best results are in \textbf{bold} and the second are \underline{underlined}.}
\vspace{-1.5ex}
\label{table:auc_acc}
\tabcolsep = 0.05cm
\begin{center}
\resizebox{\linewidth}{!}{
\begin{tabular}{lcccccccccccccccc}
\toprule
\multirow{2}{*}{\textbf{Method}} & \multicolumn{8}{c}{\textbf{Existence Prediction}}  & \multicolumn{8}{c}{\textbf{Direction Prediction}}  \\
 & \multicolumn{2}{c}{Citeseer} & \multicolumn{2}{c}{Cora} & \multicolumn{2}{c}{Epinions} & \multicolumn{2}{c}{Twitter} & \multicolumn{2}{c}{Citeseer} & \multicolumn{2}{c}{Cora} & \multicolumn{2}{c}{Epinions} & \multicolumn{2}{c}{Twitter} \\
 & AUC & ACC & AUC & ACC & AUC & ACC & AUC & ACC & AUC & ACC & AUC & ACC & AUC & ACC & AUC & ACC \\ \midrule
NERD & 80.4(0.6) & 79.0(1.1) & 85.1(0.5) & 72.8(0.1) & 78.6(0.1) & 66.5(0.1) & 94.9(0.0) & 77.6(0.1) & 82.1(1.0) & 83.1(0.6) & 92.3(0.2) & 76.2(0.5) & 86.9(0.3) & 55.6(0.1) & 95.6(0.1) & 80.9(0.1) \\
DGGAN & 84.5(0.9) & 50.1(0.0) & 89.6(0.2) & 54.6(3.0) & 84.5(2.7) & 50.0(0.0) & 99.1(0.1) & 77.8(8.1) & 89.4(0.8) & 50.0(0.0) & 95.4(0.2) & 57.1(5.0) & 94.0(2.0) & 50.0(0.0) & 99.1(0.1) & 80.7(13) \\
ODIN & 86.2(1.1) & 76.8(1.5) & 90.8(0.2) & 83.1(0.2) & \underline{90.9(0.1)} & 83.3(0.2) & \underline{99.2(0.0)} & \textbf{98.7(0.1)} & 95.4(0.8) & 87.5(1.2) & 96.7(0.2) & 90.5(0.2) & \textbf{97.4(0.0)} & 92.1(0.1) & 99.8(0.0) & 99.6(0.0) \\
DiGAE & 87.1(1.9) & 72.8(2.1) & 83.9(0.5) & 70.1(0.2) & 81.8(0.1) & 70.2(0.1) & 98.8(0.0) & 56.7(0.0) & 78.5(2.1) & 56.8(1.7) & 89.8(0.3) & 71.2(0.2) & 91.5(0.1) & 55.6(0.1) & \textbf{99.9(0.0)} & 50.0(0.0) \\
MagNet & 88.3(0.4) & 80.7(0.8) & 89.4(0.1) & 81.4(0.3) & 85.1(0.1) & 77.5(0.3) & 99.1(0.1) & 97.7(0.1) & 96.4(0.6) & 91.7(0.9) & 95.4(0.2) & 88.9(0.4) & 96.6(0.1) & 92.1(0.0) & \underline{99.9(0.1)} & 98.5(0.9) \\
SigMaNet & 91.4(0.8) & 84.4(3.2) & 93.6(0.2) & 87.7(0.4) & 90.3(0.0) & 82.5(0.0) & 99.1(0.0) & \underline{98.3(0.0)} & 98.3(0.3) & 97.8(0.9) & 96.4(0.0) & 94.7(0.1) & \underline{96.7(0.0)} & \underline{92.5(0.0)} & \textbf{99.9(0.0)} & \underline{99.7(0.0)} \\
DUPLEX* & \underline{98.0(0.7)} & \underline{95.3(0.5)} & \underline{95.8(0.3)} & \underline{93.1(0.1)} & 90.7(0.2) & \underline{84.5(0.5)} & 96.1(0.6) & 96.9(0.2) & \underline{99.3(0.5)} & \underline{98.3(0.3)} & \underline{97.1(0.2)} & \underline{95.5(0.3)} & 93.9(0.2) & 92.2(0.0) & \textbf{99.9(0.0)} & \underline{99.7(0.0)} \\
DUPLEX & \textbf{98.6(0.4)} & \textbf{95.7(0.5)} & \textbf{95.9(0.1)} & \textbf{93.2(0.1)} & \textbf{91.0(0.2)} & \textbf{85.5(0.0)} & \textbf{99.3(0.2)} & \textbf{98.7(0.1)} & \textbf{99.7(0.2)} & \textbf{98.7(0.4)} & \textbf{97.2(0.2)} & \textbf{95.9(0.1)} & 95.2(0.4) & \textbf{92.6(0.1)} & \textbf{99.9(0.0)} & \textbf{99.8(0.0)} \\
\bottomrule
\end{tabular}}
\vspace{-2.5ex}
\end{center}
\end{table*}

\begin{table*}[t]
    \centering
    \caption{Link prediction ACC (\%) for subtask 3 and 4.}
    \vspace{1.5ex}
    \label{table:acc_34}
    \tabcolsep = 0.1cm
    \resizebox{0.55\linewidth}{!}{
    \begin{tabular}{lcccccccc}
    \toprule
    \multirow{2}{*}{\textbf{Method}} & \multicolumn{4}{c}{\textbf{Three-type Classification}} & \multicolumn{4}{c}{\textbf{Four-type Classification}} \\
    & Citeseer & Cora & Epinions & Twitter  & Citeseer & Cora & Epinions & Twitter  \\ \midrule
NERD & 68.8(0.7) & 67.7(0.6) & 66.6(0.3) & 73.2(0.1) & 31.4(0.6) & 38.6(0.8) & 33.1(0.4) & 32.2(0.4) \\
DGGAN & 59.0(0.3) & 58.8(0.1) & 57.6(0.3) & 70.0(5.0) & 13.4(4.0) & 10.5(6.4) & 20.1(0.1) & 67.4(7.1) \\
ODIN & 67.2(0.8) & 72.1(0.3) & 87.3(0.1) & 98.5(0.0) & 67.1(0.8) & 70.6(0.4) & 73.1(0.1) & \underline{98.0(0.0)} \\
DiGAE & 83.7(1.3) & 68.5(0.3) & 80.4(0.2) & 69.8(0.0) & 42.3(0.6) & 34.4(0.3) & 40.2(0.2) & 35.9(0.0) \\
MagNet & 72.0(0.9) & 66.8(0.3) & 76.9(0.9) & 93.9(2.6) & 69.3(0.4) & 63.0(0.3) & 65.2(0.4) & 91.6(1.1) \\
SigMaNet & 81.3(0.4) & 80.3(0.2) & 86.4(0.1) & 98.0(0.0) & 76.2(2.1) & 78.7(0.4) & \underline{75.4(0.1)} & 97.2(0.0) \\
DUPLEX* & \underline{93.7(0.9)} & \underline{92.1(0.1)} & \underline{88.9(0.1)} & \underline{99.0(0.1)} & \underline{90.7(0.8)} & \underline{88.3(0.2)} & 74.9(0.6) & 93.6(0.3) \\
DUPLEX & \textbf{94.8(0.2)} & \textbf{92.2(0.1)} & \textbf{88.9(0.0)} & \textbf{99.2(0.1)} & \textbf{91.1(1.0)} & \textbf{88.4(0.4)} & \textbf{76.4(0.2)} & \textbf{98.1(0.2)} \\
    \bottomrule
    \end{tabular}}
    \vspace{-1ex}

\end{table*}

Embarking on our exploration with link prediction, we adopt the approach delineated by Zhang et al.\cite{Zhang2021MagNetAN} to randomly split the datasets into training, validation, and testing subsets, maintaining a ratio of 16:1:3. To thoroughly evaluate the model's proficiency in discerning varied edge types, we engage in four prevalent subtasks: \textbf{Existence Prediction} (EP)\cite{Zhu2020AdversarialDG}, which determines the likelihood of an edge's presence; \textbf{Direction Prediction} (DP)\cite{Zhu2020AdversarialDG, Zhang2021MagNetAN}, ascertaining the orientation of unidirectional edges; \textbf{Three-type Classification} (TP)\cite{Zhang2021MagNetAN}, categorizing edges as positive, reverse, or non-existent; and \textbf{Four-type Classification} (FP), discerning positive, reverse, bidirectional, or non-existent edges. We present additional details in Appendix~\ref{app:lp-setup}.

Notably, \textbf{DUPLEX}, alongside \textbf{NERD}, \textbf{DGGAN}, \textbf{DiGAE}, and \textbf{ODIN}, employs a self-supervised approach to regenerate the directed graph, enabling the direct application of learned embeddings across all subtasks. The task-specific \textbf{MagNet} and \textbf{SigMaNet}, by contrast, necessitate individual end-to-end training for each specific subtask. The outcomes, including Area Under Curve (AUC) for subtasks EP and DP, and accuracy (ACC) for all tasks, are listed in Tables~\ref{table:auc_acc}-\ref{table:acc_34}. 

DUPLEX shines in the EP and DP experiments, claiming the highest AUC in 7 out of 8 cases and dominating in ACC across all subtasks, showcasing its adeptness in capturing a spectrum of node relationships within digraphs. Even without the fusion layer, the ablated DUPLEX* variant secures top or near-top AUC or ACC scores in 12 out of 16 experiments. Almost all methods excel in DP, particularly on larger datasets, likely due to their design focus on edge directionality and the abundance of training data. Yet, DUPLEX distinguishes itself in more nuanced tasks like TP and FP, and in smaller datasets. For instance, on the Citeseer dataset, DUPLEX records approximately 11.1\% and 14.9\% improvements in TP and FP, respectively. Additionally, ODIN delivers satisfying results across most tasks, which may stem from its strategic address of potential degree distribution discrepancies between training and testing sets. Despite this, as a shallow method, ODIN requires re-initialization for newly introduced nodes and graphs. MagNet and SigMaNet also perform robustly by leveraging the magnetic Laplacian, extensively harnessing whole-graph information for embedding optimization. It should be noted, however, that they incur a greater computational complexity due to their reliance on complex-valued matrix operations, a challenge that DUPLEX circumvents with its separation of the amplitude and phase embeddings, as detailed in Appendix~\ref{app:complexity}.

\begin{figure}[t]
    \centering
    \includegraphics[width=0.49
    \textwidth]{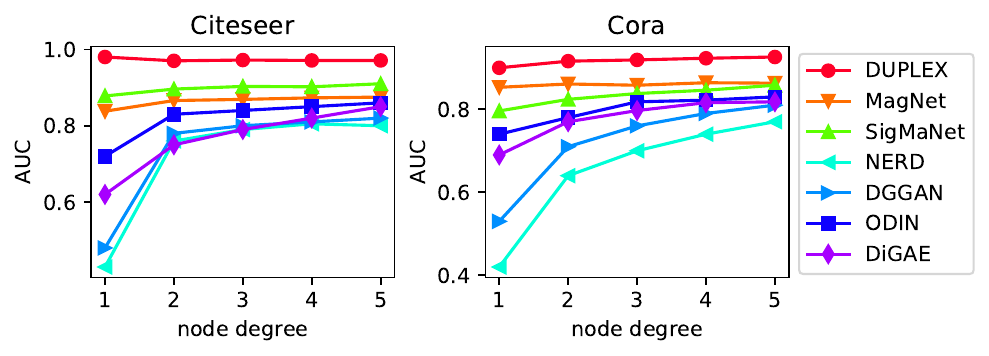}
    \vspace{-3ex}
    \caption{Link existence prediction AUC (\%) under extremely low-degree setting.}
    \label{fig:low-degree}
    \vspace{-3ex}
\end{figure}

We further plot the AUC in subtask 1 against the in/out-degree of nodes in the testing sets of Citeseer and Cora in Figure~\ref{fig:low-degree}. The x-axis in the figure denotes the threshold, indicating that in the testing set, the source node's out-degree or the target node's in-degree does not exceed the specified threshold. This analysis provides insights into the capability of the benchmark methods in handling low-degree nodes. As expected, the single embedding methods, namely DUPLEX, MagNet and SigMaNet, are robust to the change of the node degree and consistently outperform the dual embedding methods, namely NERD, DGGAN, DiGAE and ODIN. This superiority is particularly evident when the maximum in/out-degree of testing nodes becomes small (e.g., 1). In this scenario, DUPLEX's performance surges, attributable to its complex conjugate role-based collaborative optimization, yielding increases from a minimum of 25.9\% up to 55.0\% on Citeseer, and from 16.3\% to 47.8\% on Cora, when juxtaposed with dual embedding methods.
Moreover, DUPLEX performs even better than MagNet and SigMaNet. They adopt the magnetic Laplacian as the filter to extract features from the digraph and further learn a linear classifier for link prediction based on these features.
Instead, DUPLEX embraces a reconstruction-centric approach, striving to reconstitute the HAM through a direction-aware loss, subsequently guiding the dual encoder towards a deeper assimilation of the graph structure. This reconstruction-oriented methodology likely underpins DUPLEX's enhanced performance relative to feature extraction and linear classification-based methods. In summary, DUPLEX overcomes the problem of suboptimal embedding for low-degree nodes in digraphs, surpassing other methods by large margins.

\subsubsection{Transductive Node Classification}
\label{sec:tran_node_exp}

In this study, we assess the adaptability of node embeddings generated by DUPLEX, which were trained using the proposed self-supervised approach, to secondary tasks such as node classification. After initially training embeddings to encapsulate the graph's structure over the entire graph as delineated in Section~\ref{sec:method}, we leverage these embeddings to train a two-layer Multilayer Perceptron (MLP) for node classification. It is important to highlight that this experiment unfolds in a transductive setting where labels for the test set are concealed during training, but not the nodes. 

Our comparative analysis encompasses \textbf{DUPLEX} and a cohort of self-supervised counterparts, including \textbf{NERD}, \textbf{DGGAN}, \textbf{DiGAE}, and \textbf{ODIN}. For DUPLEX, we combine amplitude and phase embeddings to derive the final node representations, whereas for the remaining methods, source and target embeddings are concatenated. We also introduce \textbf{DUPLEX-S}, a supervised iteration of DUPLEX, expressly trained end-to-end for link prediction, to determine if embeddings optimized for one task maintain their relevance across others. To condition DUPLEX-S on the supervised link prediction task, we adopt the methodology from~\cite{Zhang2021MagNetAN}, which involves concatenating complex embeddings for edge-associated node pairs and employing a linear layer as both decoder and classifier. We then concurrently train the dual encoder and decoder to classify edges into four distinct categories. Nodes across all datasets are partitioned into training, validation, and testing subsets with a 3:1:1 split. We utilize macro $F_1$ and micro $F_1$-scores for our evaluation metrics.

\begin{table}[t]
\vspace{-1ex}
\caption{Transductive node classification results.} 
\vspace{1ex}
\label{table:node_trans}
\tabcolsep = 0.1cm
\begin{center}
\resizebox{0.7\linewidth}{!}{
\begin{tabular}{lcccc}
\toprule
\multirow{2}{*}{\textbf{Method}} & \multicolumn{2}{c}{\textbf{Citeseer}}   & \multicolumn{2}{c}{\textbf{Cora-ml}}    \\
                                 & mac. $F_1$   & mic. $F_1$   & mac. $F_1$   & mic. $F_1$   \\ \midrule
NERD & 49.0(0.4) & 52.5(0.7) & \textbf{78.8(2.3)} & \textbf{80.1(1.6)} \\
DGGAN & 19.1(1.5) & 22.7(1.9) & 15.6(1.4) & 23.6(3.7) \\
ODIN & 17.7(2.5) & 19.2(2.2) & 12.6(2.1) & 18.0(2.0) \\
DiGAE & 26.6(3.1) & 29.8(3.2) & 55.1(4.8) & 56.2(3.2) \\
DUPLEX-S & 40.0(0.2) & 42.6(1.1) & 64.3(1.4) & 65.7(1.2) \\
DUPLEX* & \underline{51.2(0.1)} & \underline{54.3(2.0)} & 76.0(1.3) & 77.8(1.6) \\
DUPLEX & \textbf{53.0(2.5)} & \textbf{56.2(1.7)} & \underline{77.9(0.6)} & \underline{79.8(0.7)} \\ \bottomrule
\end{tabular}}
\end{center}
\vspace{-3ex}
\end{table}

As demonstrated in Table~\ref{table:node_trans}, DUPLEX outperforms other methods on the Citeseer dataset, achieving the highest macro $F_1$ and micro $F_1$-scores and exceeding baselines by at least 4.0\% and 3.7\%, respectively. On Cora-ml, DUPLEX follows closely behind NERD, securing the second-best results. This indicates that DUPLEX’s self-supervised node embeddings are well-suited for node classification tasks. In comparison, the supervised version, DUPLEX-S, lags behind, underscoring the value of the parameter-free decoder and self-supervised learning in improving task generalization. NERD’s notable performance can be attributed to its accounting for various node similarity types (both low and high-order). Capturing higher-order proximities indeed helps preserve global network structures more effectively, as nodes with shared neighbors are likely to exhibit similarity. For example, in citation networks, a higher overlap in cited references between two papers often indicates a stronger thematic connection between them. This is particularly advantageous for classification tasks, as highlighted in previous work~\cite{tang2015line}. DUPLEX, by stacking GAT layers, captures higher-order proximities effectively, achieving comparable or even superior results to NERD. Additionally, as discussed in detail in Appendix~\ref{app:nc_perf}, NERD employs a more flexible approach to aggregating neighbor information, thus avoiding information bottlenecks and enabling better representation of low-degree nodes, thereby achieving improved node classification performance.

\subsubsection{Inductive Node Classification}
\label{sec:ind_node_exp}

We extend our assessment of DUPLEX to an inductive setting, where testing nodes are entirely unseen during the training phase, a notable departure from our earlier transductive experiment. In this experiment, we conducted a comparative analysis of DUPLEX against three groups of baseline methods. Initially, due to the absence of open-source versions of inductive directed graph embedding methods like BLADE~\cite{Virinchi2023BLADEBN}, we contrast \textbf{DUPLEX} with transductive models such as \textbf{MagNet} and \textbf{SigMaNet}. For these models, we retain a transductive framework where only node labels were hidden in the testing set, providing these methods with additional training information. Secondly, we compare DUPLEX with SOTA inductive GNNs designed for directed graph node classification, such as \textbf{Dir-GNN}~\cite{rossi2024edge}. As Dir-GNN can be integrated with different backbone models, we consider SAGE~\cite{Hamilton2017InductiveRL} (denoted as \textbf{dir-SAGE}) and GAT~\cite{Velickovic2017GraphAN} (denoted as \textbf{dir-GAT}) here. Thirdly, we pit DUPLEX against SOTA inductive models for undirected graphs, such as the initial undirected \textbf{GAT} and \textbf{SAGE}, using these comparisons as ablation studies to emphasize the importance of directional information in DUPLEX. GAT and SAGE act as simplified incarnations of DUPLEX, with GAT missing the phase encoder and SAGE serving as the undirected graph aggregator analogue. Our experiments are conducted on Citeseer and Cora-ml, both consisting of attributed graphs. As initial embeddings, we employ the node attributes, which correspond to the word embeddings of the paper descriptions in these datasets. 

\begin{table}[t]
\vspace{-1.5ex}
\caption{Inductive node classification results.} 
\vspace{1ex}
\label{table:node_ins}
\tabcolsep = 0.1cm
\begin{center}
\resizebox{0.7\linewidth}{!}{
\begin{tabular}{lcccc}
\toprule
\multirow{2}{*}{\textbf{Method}} & \multicolumn{2}{c}{\textbf{Citeseer}}   & \multicolumn{2}{c}{\textbf{Cora-ml}}    \\
                                 & mac. $F_1$   & mic. $F_1$   & mac. $F_1$   & mic. $F_1$   \\ \midrule
SAGE & 70.4(1.2) & 74.3(1.0) & 74.5(3.2) & 81.0(2.4) \\
GAT & \underline{70.8(2.0)} & \underline{74.7(1.4)} & 77.1(2.8) & 83.9(1.7) \\
dir-SAGE & 70.2(2.6) & 74.1(1.2) & 83.4(2.2) & 85.6(1.9) \\
dir-GAT & 70.4(1.6)	& 74.6(0.6) & \underline{85.0(0.9)} & 86.4(0.7) \\
MagNet & 65.0(1.7) & 74.6(1.9) & 82.0(1.7) & 84.3(1.5) \\
SigMaNet & 63.1(1.3) & 68.8(1.0) & 82.7(1.0) & 83.5(1.2) \\
DUPLEX* & 68.3(1.3) & 74.0(0.7) & \underline{85.0(2.6)} & \underline{87.3(2.5)} \\
DUPLEX & \textbf{71.7(0.7)} & \textbf{75.4(0.5)} & \textbf{85.9(0.8)} & \textbf{87.6(0.9)} \\ \bottomrule
\end{tabular}}
\end{center}
\vspace{-6ex}
\end{table}

Table~\ref{table:node_ins} details the achieved macro and micro $F_1$-scores, demonstrating DUPLEX's superior performance. Outperforming transductive methods, which benefit from a more informative training phase, DUPLEX reports increases of at least 6.7\% and 3.2\% in terms of macro $F_1$-score. In comparison to inductive GNNs, DUPLEX shows enhancements of at least 1.3\% and 0.9\% for directed graph methods and at least 0.9\% and 8.8\% for undirected graph methods. These results underscore DUPLEX's robust inductive capabilities, adeptly managing unseen nodes and accurately classifying instances across a spectrum of categories. Furthermore, DUPLEX's significant outperformance over baselines on undirected graphs emphasizes its efficiency in utilizing directional information, an attribute integral to tasks involving directed graphs.

\subsection{Ablation Study}
\label{sec:ablation}
The results of our ablation study are summarized below, with comprehensive details provided in Appendix~\ref{app:ablation}: (1) \textbf{Dual GAT Encoder}: The amplitude and phase aggregators are essential for capturing both the connections and the directions within the graph. The dual encoder confers a performance advantage irrespective of the backbone model. (2) \textbf{Fusion Layer}: The ``mid-fusion'' approach is preferred for graphs lacking attributes; ``Fusion aggregation'' in Eq.~\eqref{eq:am_fu_update} demonstrates superioriority over a strategy that involves two distinct steps. (3) \textbf{Sensitivity Analysis}: The inclusion of the connection-aware loss can enhance the performance, but the enhancement decays when the weight $\lambda$ increases.

\section{Conclusion}
We propose DUPLEX for digraph embedding that yields high-quality embeddings especially for low in/out-degree nodes, adapts well to diverse downstream tasks, and generalizes effectively to unseen nodes. DUPLEX achieves these objectives through a dual encoder, two parameter-free decoders, and two self-supervised losses. Results show that it surpasses SOTA methods in link prediction, as well as transductive and inductive node classification tasks.

\section*{Acknowledgements}
We would like to thank Ant Group for their support for this work.

\section*{Impact Statement}
This paper presents work whose goal is to advance the field of graph neural networks. There are many potential societal consequences of our work, though we believe it is unnecessary to highlight them in this context.

\bibliography{example_paper}
\bibliographystyle{icml2024}
\clearpage
\appendix
\onecolumn

\section{Symbols and Notations}

\begin{table*}[h]
\begin{center}
\caption{Symbols and notations used in this paper.}
\resizebox{0.5\linewidth}{!}{
\begin{tabular}{ll}
\toprule
Symbols         & Notations                                                                    \\ \midrule
$\gG$           & a digraph with nodes and edges                                        \\
$\gV$           & the set of nodes in the graph                                                \\
$\gE$           & the set of edges in the graph                                                \\
$u,v$           & two nodes on the graph for illustration                                       \\
$\mA$           & the adjacency matrix of graph                                                \\
$\vx_u$         & the embedding vector of nodes $u$                                                   \\
$\mX$           & the embedding matrix of nodes                                                \\
$\mH$           & the Hermitian adjacency matrix of graph                                      \\
$\mA_s$         & the amplitude matrix of HAM                                                   \\
$\mTheta$       & the phase matrix of HAM                                                          \\
$\va_u$         & the amplitude embedding of nodes $u$                                           \\
$\vtheta_u$     & the phase embedding of nodes $u$                                               \\
$\gN(u)$        & the neighbors of node $u$                                                    \\
$\sigma(\cdot)$ & the sigmoid function\\
$\phi(\cdot)$ & the activation function \\
$\gR$           & four relation types of edge                                                  \\
$r$             & a certain relation type of edge                                              \\
$q$  & the decay rate of the weight of connection-aware loss\\
$\lambda$ & the initial weight of the connection-aware loss\\
$d$ & embedding dims\\

\bottomrule
\end{tabular}}
\end{center}
\end{table*}

\section{Limitations of Dual Embedding Methods: An SVD Perspective} 
\label{app:SVD_drawback} 

We address a specific challenge that dual embedding methods face with nodes exhibiting low in-/out-degrees and rationalize this from a Singular Value Decomposition (SVD) standpoint. In this section, we introduce Lemma~\ref{le:rank} (see~\cite{hohn2013elementary}) and Lemma~\ref{le:myle}, provide a proof of Lemma~\ref{le:myle}, and summarize our insights.
\begin{lemma}[Sylvester's rank inequality] \label{le:rank} For matrices $\mA\in \sR^{m\times n}$ and $\mB\in \sR^{n\times s}$ with $\mA\mB=\m0$, it holds that $r(\mA) + r(\mB)\leq n$. \end{lemma}
In adjacency matrices, nodes with zero out-degree correspond to rows filled with zeros, while zero in-degree nodes match with columns of zeros. Focusing on the zero out-degree nodes, we establish that:
\begin{lemma}
\label{le:myle}
    If an asymmetric matrix has several all-zero rows, the corresponding rows in its source embedding matrix are all-zero.
\end{lemma}
\begin{proof}
Let $\mA$ be an asymmetric matrix with several zero rows. We can represent it as follows by renumbering nodes, that is, by interchanging the rows and columns:
\begin{equation}
    \mA = \begin{bmatrix} \mA'\\\m0 \end{bmatrix},
\end{equation}
where $\mA'$ is a submatrix of $\mA$ consisting of all non-zero rows. We further perform the singular value decomposition (SVD) on $\mA$:
\begin{equation}
    \begin{bmatrix} \mA'\\\m0 \end{bmatrix} = \begin{bmatrix} \mU_1\\\mU_2 \end{bmatrix}\Sigma \mV^T,
\end{equation}
where
\begin{equation}
\label{eq:u1}
    \mA' = \mU_1\Sigma \mV^T,
\end{equation}
and
\begin{equation}
    \m0 = \mU_2\Sigma \mV^T.
\end{equation}

It is important to note that $\mA'$ is the predominantly optimized component, encompassing the majority of the training samples. Specifically, our emphasis lies on optimizing Eq.~\eqref{eq:u1} by adjusting its singular values and vectors. In the context of low-rank decomposition, we consider the first $d$ largest singular values, which are both non-zero and distinct, thus the corresponding singular vectors are orthogonal.

With Lemma~\ref{le:rank}, we have:
\begin{equation}
r(\mU_2)+r(\Sigma \mV)<=d, 
\end{equation}
where $r(\Sigma \mV)=d$ as the column vectors of $\mV$ are orthogonal. Consequently, we deduce that $r(\mU_2)=\m0$, that is,
\begin{equation}
    \mU_2=\m0.
\end{equation}
For the dual embedding methods, the source embedding matrix can be written as
\begin{equation}
    \mX_s=\begin{bmatrix} \mU_1\\\mU_2 \end{bmatrix}\Sigma^{1/2}=\begin{bmatrix} \mU_1\\\m0 \end{bmatrix}\Sigma^{1/2},
\end{equation}
that is, the corresponding rows in the source embedding matrix are all-zero.
\end{proof}

Similarly, it can be shown that an asymmetric matrix with columns of zeros will have target embedding matrix columns filled with zeros, leading to zero target/source embeddings for nodes with zero in-/out-degrees.
Examining nodes with low in-/out-degrees, and taking low out-degree nodes as an example, the adjacency matrix is divided into high and low out-degree segments: 
\begin{equation} 
\mA = \begin{bmatrix} \mA'\\ \mO \end{bmatrix}, 
\end{equation} 
with $\mA'$ representing high out-degree nodes and $\mO$ for low out-degree nodes. Suppose that errors are uniformly distributed across both segments during optimization (i.e., reconstruction of $\mA$), and these errors set non-zero entries in $\mA$ to zero. This affects low-degree nodes more significantly, potentially pushing their embeddings toward zero and causing them to cluster in the embedding space, as cautioned by Lemma~\ref{le:myle}.
Although this issue is not rigorously proven for low-degree nodes and remains a topic for further exploration, we can assert that the HAM does not suffer from this problem as it has all-zero rows or columns only for isolated nodes with both zero in and out-degree, which are outside the scope of our analysis.

\section{Graph Attention Networks}
GAT~\cite{Velickovic2017GraphAN} is a successful practice of incorporating attention mechanisms into graph neural networks. Specifically, GAT assigns attention coefficients $\alpha_{uv}$ to each connected pair of nodes $(u,v)$ as weights during neighbor aggregation. In DUPLEX, we use GAT as our backbone model for both amplitude and phase encoders. Taking the amplitude GAT as an example, we further write Eq.~\eqref{eq:am_agg} as:
\begin{equation}
    \label{eq:att}
    \sum_{v\in\gN(u)}f_a(\va_u,\va_v)\psi_a(\va_v) = \sum_{v\in{\gN(u)}}\alpha_{uv}\mW_a\va_v,
\end{equation}
where the attention coefficient,
\begin{equation}
    \label{eq:att_alpha}
    \alpha_{uv} = \frac{\exp(\text{LeakyReLU}(\vb_a^\intercal[\mW_a\va_u||\mW_a\va_v]))}{\sum_{g\in \gN(u)}\exp(\text{LeakyReLU}(\vb_a^\intercal[\mW_a\va_u||\mW_a\va_g]))}.
\end{equation}
Here, $\text{exp}(\cdot)$ is the exponential operation, $\text{LeakyReLU}(\cdot)$ is the activation function, $||$ is the concatenation operation, $\vb_a\in \sR^{2d\times 1}$ is the attention parameters of the amplitude encoder, and $\mW_a\in \sR^{d\times d}$ is the feature transformation parameters.

\section{Complexity Analysis}
\label{app:complexity}
Assuming the network has $L$ layers in total, the embedding dimension is $d$, with $|\gV|$ nodes and $|\gE|$ edges in the graph. In our method, for each head of the attention mechanism, the time complexity of the encoding process is $O(L|\gV|d^2+L|\gE|d)$, and the space complexity is $O(|\gE|+Ld^2+L|\gV|d+Ld)$. The time complexity of the decoding process is $O(|\gE|d)$, and the space complexity is $O(|\gV|d+|\gE|)$. The total number of parameters in the model is $2Ld^2+4Ld$.

We further compare the dual GAT encoder of DUPLEX with a single GAT encoder that directly produces a $d$-dimensional vector of complex numbers for each node in each layer, where the complex vector is often converted to a $2d$-dimensional real-valued vector without separating the amplitude and phase. The complexity of feature transformation in each layer of the dual GAT encoder, encompassing both amplitude and phase transformations, amounts to $2|\gV|d^2$. In contrast, for the single GAT encoder, the complexity for the $|\gV|\times 2d$ feature matrix is $4|\gV|d^2$. This reveals that our dual GAT encoder reduces feature transformation complexity by half, while maintaining the same complexity for other operations.

\section{More Details on Datasets}
\label{app:datasets}
We perform our experiments on five open-source digraph datasets, including Cora-ml, Citeseer, Cora, Epinions and Twitter. \textbf{Citeseer} and \textbf{Cora} are two popular citation networks, \textbf{Cora-ml} is a subset of Cora dataset in the field of machine learning. The Cora-ml and Citeseer datasets provide meaningful node features, with the node labels corresponding to scientific subareas. \textbf{Epinions} and \textbf{Twitter} are two social networks. We use the version of \textbf{Citeseer} and \textbf{Cora-ml} dataset provided by~\cite{Zhang2021MagNetAN}, \textbf{Cora} dataset provided by~\cite{ubelj2013ModelOC} and other two datasets from the Stanford Large Network Dataset Collection\footnote{\url{https://snap.stanford.edu/data/}}. The overview of the datasets is shown in Table~\ref{table:dataset}. The graph datasets used in our study are unweighted and directed. These datasets encompass a wide range of graph sizes, with the number of nodes varying from $3K$ to $465K$, and the number of edges ranging from $4K$ to $834K$.

The direction information is quite important in these datasets. For example, consider two nodes in a citation network. Node \textit{a} has many in-neighbors, indicating that it is a popular paper cited by many other papers. On the other hand, node \textit{b} has many out-neighbors, indicating that it is a knowledgeable paper that cites many other papers. These two papers have completely different characteristics. However, when using undirected graph embedding methods, as both nodes are connected to many other nodes, they would be considered to have similar structural features in the graph, which can mislead downstream tasks.
Similarly, in a social network, if A follows several people who all follow B, it is more likely that A will follow B rather than the other way around. Undirected graph embedding methods would consider A and B to have similar structural features and fail to distinguish their asymmetrical relationship, leading to potential errors in recommendations and other tasks.

\section{Implementation Details}
\label{app:implementation}
We compare DUPLEX with several SOTA models, where NERD, DGGAN, DiGAE and ODIN learn a dual embedding while DUPLEX, MagNet and SigMaNet learn a single embedding for each node. During experiments, the dual embedding methods learn two 128-dimensional real-valued embeddings for each node, while the single embedding methods learn a 128-dimensional complex embedding for each node.

We implemented DUPLEX using DGL and PyTorch, employing Adam optimizer with a learning rate of 1e-3. DUPLEX consists of two 3-layer GATs with attention head 1, one for the amplitude embedding and the other for the phase embedding. We sampled four types of node pairs (forward edge, reverse edge, bidirectional edge, and no edge) at a ratio 1:1:1:$x$ for self-supervised loss computation, where $x$ tends to be small due to fewer bidirectional edges, and in our experiments we just control $x \leq 1$. We set the hidden dim to 128 and the dropout rate to 0.5. We tuned the initial loss weight $\lambda$ in \{0.1,0.3\} and the decay rate $q$ in \{0,1e-4,1e-2\}. We ran our model with maximum 3000 epochs with early stopping for all experiments.

For the baseline methods, we utilize the publicly available code repositories, and tune the hyperparameters as follows. 
\begin{itemize}[leftmargin=1em,itemsep= 1pt,topsep = 1pt,partopsep=0pt]
    \item NERD is a random-walk based model, for which we tune the walk size in \{1, 4, 7\}, the start learning rate in \{0.001, 0.01, 0.1\} and the negative samples in \{1, 4, 7\}.
    \item DGGAN is a GAN-based model, for which we set the default parameters as the paper.
    \item ODIN is a shallow method incorporating multiple losses with distinct objectives, where we tune the disentangling loss weights in \{0.3, 0.5, 0.7\} and the negative sampling rates in \{1, 4, 7\}.
    \item MagNet is a GCN-based model, where we set the GCN layer as 3, hidden dim as 128, and tune the parameter q, which controls the importance of directed information, in \{0.05, 0.1, 0.15, 0.2, 0.25\}.
    \item DiGAE is a digraph autoencoder model which set directed GCN as the encoder, for which we tune the hyperparameter $\alpha, \beta$ in \{0.3, 0.5, 0.7\}.
    \item SigMaNet is a GCN-based model, where we set the GCN layer as 3 and the hidden dim as 128.
\end{itemize}

\section{Link Prediction Set-up}
In Section~\ref{sec:lp}, we assess the model's ability to discriminate between different edge types and consider four commonly-used subtasks. We list the details of each subtask here.
\label{app:lp-setup}
\begin{enumerate}[leftmargin=1em,itemsep= 1pt,topsep = 1pt,partopsep=0pt]
 \item Existence Prediction (EP)~\cite{Zhu2020AdversarialDG}: This subtask involves predicting whether a specific edge exists in the graph or not. In the testing sets, the ratio of edges belonging to the two types is 1:1. Note that both reverse and non-existent edges are categorized as non-existence~\cite{Zhu2020AdversarialDG}, and their ratio is also 1:1 in the testing set.

 \item Direction Prediction (DP)~\cite{Zhu2020AdversarialDG, Zhang2021MagNetAN}: The objective of this subtask is to predict the direction of unidirectional edges (i.e., forward or reverse), with a ratio of 1:1 in the testing sets.

 \item Three-type Classification (TP)~\cite{Zhang2021MagNetAN}: This subtask involves classifying edges into three types: positive, reverse, or non-existent, with a ratio of 1:1:1 in the testing set.

 \item Four-type Classification (FP): In this subtask, the goal is to classify edges into four categories: positive, reverse, bidirectional, or non-existent. The ratio for the four types of edges in the testing set is set to 1:1:1:1. In cases where the number of bidirectional edges is smaller than the other types in certain datasets, we utilize all available bidirectional edges in the testing set.
\end{enumerate}

\section{Node Classification Performance}
\label{app:nc_perf}
The ease of classifying nodes is closely tied to their degrees. It is widely acknowledged that nodes with ample accessible information, whether emanating from structural linkages or intrinsic nodal attributes, are typically more amenable to effective classification. When nodes are bolstered by reliable attributes, these specific individualistic features can significantly diminish the relative importance of structural information in achieving separability between different node classes. Consequently, even nodes with a low degree have the potential to be accurately classified based solely on their distinctive attributes. In contrast, when such nodal attributes are absent or lack distinctive information, the graph's structural properties gain prominence as a critical discriminative element. In this context, nodes characterized by a scant number of connections, or a low in/out degree, inherently face challenging classification scenarios attributed to their sparse connectivity within the overall graph structure.

Addressing these challenges, DUPLEX brings forth a flexible encoding approach that dynamically assimilates information within the graph. Recall that one common embedding method for directed graphs employs the source and target embeddings~\cite{Khosla2018NodeRL, Kollias2022DirectedGA, Yoo2023DisentanglingDB}. To finish the node classification task, the concatenation of the source and target embeddings is used, that is, $[\vs,\vt]$. We notice that typically $\vs = f(\vx_{\text{out}})$ and $\vt = f(\vx_{\text{in}})$. In other words, the source embeddings only aggregate information from the out-neighbors, while the target embeddings only aggregate information from the in-neighbors. Information from both types of neighbors can only be exchanged in the final concatenation stage. By contrast, DUPLEX's design innovatively facilitates information fusion between both types of neighbors at every layer. This occurs through the coordinated use of its amplitude encoder with an undirected graph aggregator and its phase encoder with a directed graph aggregator, circumventing the information bottleneck commonly found in previous approaches.

Furthermore, since we employ the attention mechanism for the sake of aggregation, it can learn the attention score between two nodes adaptively from the data. When this score approaches zero, the information exchange is blocked. For graphs with reliable attributes, DUPLEX possesses the capability to truncate the exchange of information between nodes by modulating attention scores that approach zero and sorely relies on the node attributes to achieve good classification performance. Conversely, in the case when the graph structure plays a more important role, DUPLEX can aggregate information from both in and out neighbors by maintaining higher attention scores, aiding low-degree nodes in overcoming their inherent informational disadvantage.

\section{The Convergence of Self-supervised Training}
\label{app:loss}

The training of DUPLEX via self-supervised loss relies on the principles of stochastic gradient descent (SGD), a well-established optimization method with established theoretical foundations ensuring convergence under mild conditions. Specifically, based on the theory of stochastic gradient descent~\cite{robert1999monte}, using a schedule of the learning rates ${\rho_t}$ such that $\sum\rho_t = \infty$ and $\sum\rho_t^2 < \infty$, the training process will converge to a local minimum of the DUPLEX loss function.

In addition to this theoretical assurance, we capture and illustrate the progression of the loss function's convergence over the course of the training period through visual plots. Likewise, we track the mean square error between the HAM approximated by DUPLEX and the actual ground truth HAM, with these findings detailed in Figure~\ref{fig:loss}. The graphical representations of these metrics clearly depict the effective convergence of the model during training, which corroborates the theoretical underpinnings of our approach.

\begin{figure*}[h]
\begin{center}
\includegraphics[width=0.5
\textwidth]{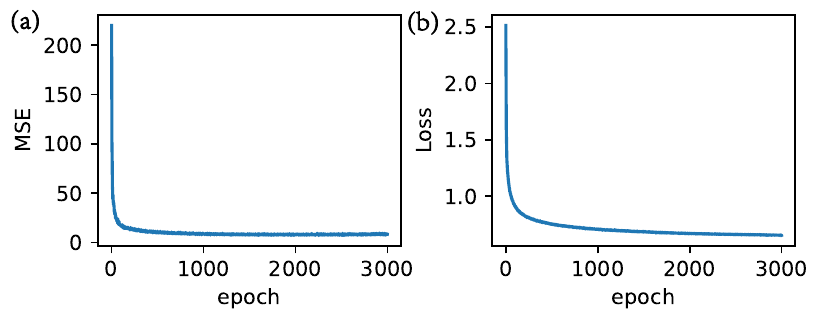}
\end{center}
\vspace{-2ex}
\caption{(a) Loss curve. (b) Mean square error between approximated HAM and the ground truth.}
\label{fig:loss}
\end{figure*}

\section{Ablation Study}
\label{app:ablation}
In this section, we perform ablation studies to delve into analyzing the role of each component in DUPLEX.

\subsection{Dual GAT Encoder}
\label{app:dual_gat_encoder}

We consider two variants of the DUPLEX: \textbf{DUPLEX-am}, which employs the undirected graph aggregator $\oplus$ (see Eq.~\eqref{eq:am_agg}) for both amplitude and phase embeddings, and \textbf{DUPLEX-ph}, which instead utilizes the digraph aggregator $\ovset{\rightharpoonup}{\oplus}$ (see Eq.~\eqref{eq:ph_agg}) for both types of embeddings. We conduct the four subtasks of link prediction on the Citeseer and Cora datasets, following the same settings in Section~\ref{sec:lp}. As shown in Table~\ref{table:dual_gnn}, DUPLEX outperforms both DUPLEX-am and DUPLEX-ph, underscoring the necessity of using the directed and undirected graph aggregator respectively for phase and amplitude embeddings. DUPLEX-am performs poorly in all tasks due to the omission of direction. Interestingly, DUPLEX-ph performs well in the direction prediction subtask, suggesting the effectiveness of the proposed directed aggregator for capturing the direction information.

\begin{table}[h]
\caption{Link prediction ACC and AUC (\%) of four subtasks. The best results are in \textbf{bold} and the second are \underline{underlined}.
}
\label{table:dual_gnn}
\begin{center}
\resizebox{0.6\linewidth}{!}{
\begin{tabular}{llcccccc}
\toprule
\multirow{2}{*}{\textbf{Dataset}} & \multirow{2}{*}{\textbf{Method}} & \multicolumn{2}{c}{\textbf{EP}}     & \multicolumn{2}{c}{\textbf{DP}}     & \textbf{TP}    & \textbf{FP}    \\
                                  &                                  & AUC                & ACC                & AUC                & ACC                & ACC                & ACC                \\ \midrule
\multirow{3}{*}{Citeseer} & DUPLEX-am & 86.2(3.9) & 79.3(5.1) & 93.2(1.1) & 88.4(1.9) & 73.4(2.7) & 67.7(3.0) \\
 & DUPLEX-ph & \underline{93.7(1.2)} & \underline{91.8(1.2)} & \textbf{99.8(0.0)} & \textbf{99.0(0.1)} & \underline{82.9(1.0)} & \underline{80.8(1.0)} \\
 & DUPLEX & \textbf{98.6(0.4)} & \textbf{95.7(0.5)} & \underline{99.7(0.2)} & \underline{98.7(0.4)} & \textbf{94.8(0.2)} & \textbf{91.1(1.0)} \\ \midrule
\multirow{3}{*}{Cora} & DUPLEX-am & 89.5(0.2) & 87.0(0.2) & 88.8(0.8) & 86.7(0.3) & 85.1(1.3) & 81.8(1.8) \\
 & DUPLEX-ph & \underline{95.2(0.1)} & \underline{91.2(0.1)} & \textbf{98.4(0.2)} & \textbf{97.2(0.2)} & \underline{88.3(0.2)} & \underline{83.7(0.7)} \\
 & DUPLEX & \textbf{95.9(0.1)} & \textbf{93.2(0.1)} & \underline{97.2(0.2)} & \underline{95.9(0.1)} & \textbf{92.2(0.1)} & \textbf{88.4(0.4)} \\ \bottomrule
\end{tabular}}
\end{center}
\end{table}

Moreover, to demonstrate the robustness of the dual encoder design with respect to different backbone GNNs, we replace the GAT backbone with the spatial GCN~\cite{kipf2017semisupervised, Hamilton2017InductiveRL}, where the aggregation function is:
\begin{equation}
 \label{eq:am_agg_gcn}
  \oplus\big(\{\va_v,\forall v\in\gN(u)\}\big) = \sum_{v\in\gN(u)}\psi_a(\va_v),
\end{equation}
and
\begin{equation}
    \label{eq:ph_agg_gcn}
    \ovset{\rightharpoonup}{\oplus}\big(\{\vtheta_v,\forall v\in\gN(u)\}\big)=\sum_{v\in\gN_{\text{in}}(u)}\psi_\theta(\vtheta_v)-\sum_{v\in\gN_{\text{out}}(u)}\psi_\theta(\vtheta_v).
\end{equation}
The difference is that there is no need for computation of attention coefficients. 

We repeat our experiments concerning link prediction, transductive and inductive node classification. The experiment results are shown in Table~\ref{table:LP-backbone}-\ref{table:NCI-backbone}. It can be observed that DUPLEX with a GAT backbone consistently outperforms the counterpart with a GCN backbone, due to the higher flexibility of GAT. However, even with the GCN backbone, DUPLEX still surpasses other SOTA methods across various tasks, highlighting the advantage of the dual encoder design regardless of the backbone. 

\begin{table*}[h]
\caption{Link prediction ACC (\%) for four subtasks. The best results are in \textbf{bold} and the second are \underline{underlined}. Note that `EP' represents the `Existence Prediction' subtask, while `DP' for direction Prediction, `TP' for three-type classification and `FP' for four-type classification.}
\label{table:LP-backbone}
\tabcolsep = 0.1cm
\begin{center}
\resizebox{0.65\linewidth}{!}{\begin{tabular}{lcccccccc}
\toprule
\multirow{2}{*}{\textbf{Method}} & \multicolumn{4}{c}{\textbf{Cora}}                                                 & \multicolumn{4}{c}{\textbf{Epinions}}                                             \\
                                 & EP                 & DP                 & TP                 & FP                 & EP                 & DP                 & TP                 & FP                 \\ \midrule
DUPLEX*(GCN) & 92.2(0.3) & 95.2(0.2) & 90.8(0.2) & 86.8(0.0) & 83.3(0.2) & 91.8(0.1) & 87.3(0.5) & 73.1(0.1) \\
DUPLEX*(GAT) & \underline{93.1(0.1)} & 95.5(0.3) & \underline{92.1(0.1)} & \underline{88.3(0.2)} & 84.5(0.5) & 92.2(0.0) & 88.9(0.1) & 74.9(0.6) \\ \midrule
DUPLEX(GCN) & 93.0(0.1) & \underline{95.6(0.1)} & 91.9(0.3) & 88.1(0.4) & \underline{84.7(0.2)} & \underline{92.5(0.0)} & \underline{87.6(0.3)} & \underline{75.2(0.2)} \\
DUPLEX(GAT) & \textbf{93.2(0.1)} & \textbf{95.9(0.1)} & \textbf{92.2(0.1)} & \textbf{88.4(0.4)} & \textbf{85.5(0.0)} & \textbf{92.6(0.1)} & \textbf{88.9(0.0)} & \textbf{76.4(0.2)} \\ \bottomrule
\end{tabular}}
\end{center}
\vspace{-3ex}
\end{table*}

\begin{table*}[h]
\begin{minipage}[c]{0.47\linewidth}
    \centering
    \caption{Transductive Node classification result (\%) with self-supervised training. }
    \vspace{1ex}
    \label{table:NCT-backbone}
    \tabcolsep = 0.1cm
    \resizebox{0.8\linewidth}{!}{\begin{tabular}{lcccc}
    \toprule
    \multirow{2}{*}{\textbf{Method}} & \multicolumn{2}{c}{\textbf{Citeseer}}   & \multicolumn{2}{c}{\textbf{Cora-ml}}    \\
                                     & mac. $F_1$   & mic. $F_1$   & mac. $F_1$   & mic. $F_1$   \\ \midrule
    DUPLEX*(GCN) & 49.0(1.3) & 52.1(1.1) & 73.9(0.3) & 76.2(0.3) \\
    DUPLEX*(GAT) & \underline{51.2(0.1)} & \underline{54.3(2.0)} & \underline{76.0(1.3)} & \underline{77.8(1.6)} \\ \midrule
    DUPLEX(GCN) & 51.2(1.0) & 54.0(1.0) & 75.5(0.7) & 77.7(0.5) \\
    DUPLEX(GAT) & \textbf{53.0(2.5)} & \textbf{56.2(1.7)} & \textbf{77.9(0.6)} & \textbf{79.8(0.7)}    \\ \bottomrule
    \end{tabular}}
    \end{minipage}\hfill
    \begin{minipage}[c]{0.47\linewidth}
        \centering
        \caption{Inductive node classification macro $F_1$ (\%) and micro $F_1$ (\%) with supervised training.}
        \vspace{1ex}
        \label{table:NCI-backbone}
        \tabcolsep = 0.1cm
        \resizebox{0.8\linewidth}{!}{\begin{tabular}{lcccc}
        \toprule
        \multirow{2}{*}{\textbf{Method}} & \multicolumn{2}{c}{\textbf{Citeseer}}   & \multicolumn{2}{c}{\textbf{Cora-ml}}    \\
                                         & mac. $F_1$   & mic. $F_1$   & mac. $F_1$   & mic. $F_1$   \\ \midrule
        DUPLEX*(GCN) & \underline{70.9(0.6)} & \underline{74.9(0.9)} & 82.2(1.5) & 84.8(1.1) \\
        DUPLEX*(GAT) & 68.3(1.3) & 74.0(0.7) & \underline{85.0(2.6)} & \underline{87.3(2.5)} \\ \midrule
        DUPLEX(GCN) & 68.9(2.7) & 74.7(0.9) & 82.6(2.1) & 85.0(1.1) \\
        DUPLEX(GAT) & \textbf{71.7(0.7)} & \textbf{75.4(0.5)} & \textbf{85.9(0.8)} & \textbf{87.6(0.9)} \\
        \bottomrule
        \end{tabular}}
\end{minipage}
\vspace{-3ex}
\end{table*}

\subsection{Fusion Layer}
\label{app:fusion}

We can tell that the fusion layer can enhance the performance of DUPLEX across all tasks and datasets by comparing DUPLEX with DUPLEX* in Tables~\ref{table:auc_acc}-\ref{table:node_ins}, implying the importance of utilizing the complementary information in the amplitude and phase encoder. Despite this, DUPLEX* still achieves commendable results, competing favorably against benchmark methods. This empirical outcome underscores the robustness of DUPLEX's design, even when the encoders operate in isolation without information fusion.

We further investigate the impact of the placement of the fusion layer on the performance of DUPLEX, considering four strategies: \textbf{early-fusion}, \textbf{mid-fusion}, \textbf{late-fusion}, and \textbf{all-fusion}. These strategies involve applying the fusion operation defined in Eq.~\eqref{eq:am_fu_update} at the input, middle, output, and all layers, respectively. Our investigation includes both non-attributed and attributed graphs, focusing on the task of node classification. Notably, in non-attributed graphs, the node embeddings are randomly initialized and trained in a self-supervised manner, while in attributed graphs, they are initialized using the node attributes and trained in a fully supervised manner. The results are shown in Figs.~\ref{fig:trans_layer}-\ref{fig:ind_layer}. For non-attributed graphs, the mid-fusion strategy achieves the best performance, in agreement with the analysis discussed in Section~\ref{ssec:fusion_layer}. On the other hand, for attributed graphs, the all-fusion strategy yields the highest performance. In this scenario, early-fusion does not introduce noise, as the node attributes hold meaningful information. Similarly, late-fusion does not compromise performance, given that the model is fully supervised for node classification, and the independence of amplitude and phase embeddings is not necessary. Consequently, fusion at all layers effectively integrates the valuable node attributes with the graph structure, resulting in improved performance.

\begin{figure*}[h]
\begin{center}
\includegraphics[width=0.85
\textwidth]{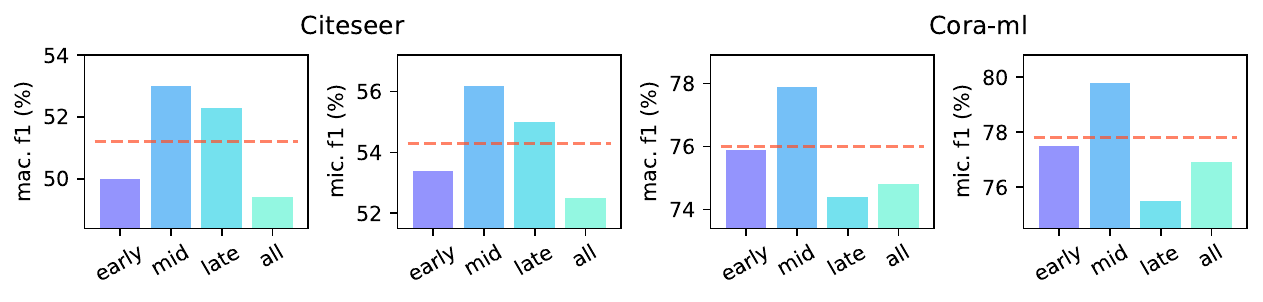}
\end{center}
\vspace{-2ex}
\caption{Node classification macro $F_1$ (\%) and micro $F_1$ (\%) with self-supervised training on randomly initialized graphs. The `early' represent for the `early-fusion' strategy, `mid' for `mid-fusion', `late' for `late-fusion', `all' for `all-fusion'. The dashed line is the baseline result with no fusion layer.}
\label{fig:trans_layer}
\end{figure*}

\begin{figure*}[h]
\begin{center}
\includegraphics[width=0.85
\textwidth]{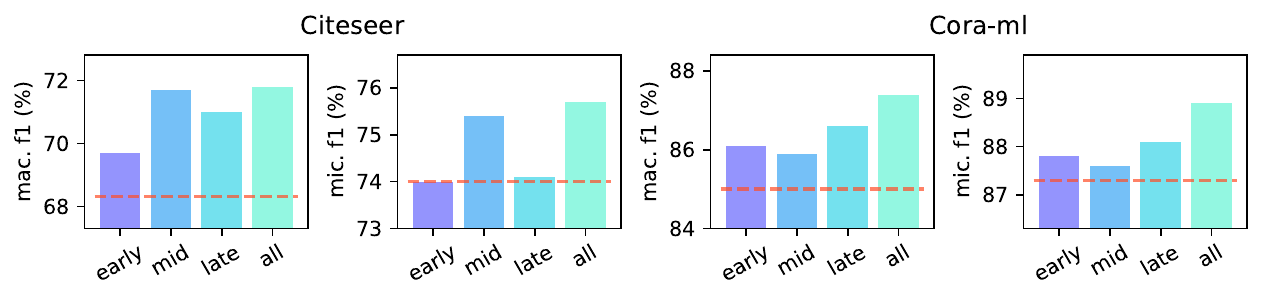}
\end{center}
\vspace{-2ex}
\caption{Node classification macro $F_1$ (\%) and micro $F_1$ (\%) with supervised training on feature-initialized graphs. The `early' represent for the `early-fusion' strategy, `mid' for `mid-fusion', `late' for `late-fusion', `all' for `all-fusion'. The dashed line is the baseline result with no fusion layer.}
\label{fig:ind_layer}
\end{figure*}

Additionally, to see whether integrating fusion with aggregation behaves better than separating fusion and aggregation operations, we explore another commonly-used fusion approach, known as the element-wise sum \cite{Wu2018MultimodalCF, Yu2017MultimodalFB}. This approach computes a weighted combination of the amplitude embedding and phase embedding of each layer: $a_u'=\phi\big(a_u+\psi_{a\theta}(\theta_u\big))$. Unlike our fusion method, this approach is independent of the graph structure and fuses the amplitude and phase embeddings of the same node after aggregation in each layer. The results for both transductive and inductive node classification can be found in Table~\ref{table:NCT-fusion}-\ref{table:NCI-fusion}. We can see that DUPLEX with the proposed fusion method consistently outperforms the element-wise sum approach. This superiority arises from its ability to better exploit the graph structure during fusion, facilitating a more effective exchange of information.

\begin{table*}[h]
\vspace{-0ex}
\begin{minipage}[c]{0.47\linewidth}
    \centering
    \caption{Transductive Node classification result (\%) with self-supervised training. The DUPLEX(EMS) represents DUPLEX with element-wise sum fusion.}
    \vspace{1ex}
    \label{table:NCT-fusion}
    \tabcolsep = 0.1cm
    \resizebox{0.8\linewidth}{!}{\begin{tabular}{lcccc}
    \toprule
    \multirow{2}{*}{\textbf{Method}} & \multicolumn{2}{c}{\textbf{Citeseer}}   & \multicolumn{2}{c}{\textbf{Cora-ml}}    \\
                                     & mac. $F_1$   & mic. $F_1$   & mac. $F_1$   & mic. $F_1$   \\ \midrule
    DUPLEX* & 51.2(0.1) & 54.3(2.0) & \underline{76.0(1.3)} & \underline{77.8(1.6)} \\
    DUPLEX(EWS) & \underline{52.7(1.1)} & \underline{56.1(2.0)} & 74.6(2.9) & 76.7(2.5) \\
    DUPLEX & \textbf{53.0(2.5)} & \textbf{56.2(1.7)} & \textbf{77.9(0.6)} & \textbf{79.8(0.7)} \\ \bottomrule
    \end{tabular}}

    \end{minipage}\hfill
    \begin{minipage}[c]{0.47\linewidth}
    \centering
    \caption{Inductive node classification macro $F_1$ (\%) and micro $F_1$ (\%) with supervised training. The DUPLEX(EMS) represents DUPLEX with element-wise sum fusion.}
    \vspace{1ex}
    \label{table:NCI-fusion}
    \tabcolsep = 0.1cm
    \resizebox{0.8\linewidth}{!}{\begin{tabular}{lcccc}
    \toprule
    \multirow{2}{*}{\textbf{method}} & \multicolumn{2}{c}{\textbf{Citeseer}}   & \multicolumn{2}{c}{\textbf{Cora-ml}}    \\
        & mac. $F_1$   & mic. $F_1$   & mac. $F_1$   & mic. $F_1$   \\ \midrule
    DUPLEX* & \underline{68.3(1.3)} & \underline{74.0(0.7)} & \underline{85.0(2.6)} & \underline{87.3(2.5)} \\
    DUPLEX(EMS) & 57.1(1.0) & 62.7(1.0) & 66.6(2.4) & 67.8(1.8) \\
    DUPLEX & \textbf{71.7(0.7)} & \textbf{75.4(0.5)} & \textbf{85.9(0.8)} & \textbf{87.6(0.9)} \\ \bottomrule
    \end{tabular}}
\end{minipage}
\end{table*}

\subsection{Sensitivity Analysis}
\label{app:sensitivity}

To check the sensitivity of the initial loss weight $\lambda$ and the decay rate $q$ in the loss function (see Section~\ref{sec:ttloss}), we conduct experiments by varying the loss weight $\lambda$ from 0.0 to 1.0 and the decay rate $q$ in the set \{0, 1e-4, 1e-2\}. The results for link prediction and transductive node classification are presented in Figure~\ref{fig:sensi_edge}-\ref{fig:sensi_node}. Our findings indicate that incorporating a non-zero weight $\lambda > 0$ for the connection-aware loss results in improved performance compared to having no connection-aware loss ($\lambda = 0$). However, as the initial weight $\lambda$ increases, the performance of DUPLEX begins to decrease. Additionally, increasing the decay rate $q$ reduces the influence of different initial weights on the model's performance. These observations suggest that the connection-aware loss aids in the learning of superior embeddings. However, setting the initial loss weight too high can interfere with the primary classification tasks, as the optimization objectives for the two losses are not entirely aligned.

\begin{figure*}[h]
\begin{center}
\includegraphics[width=0.85
\textwidth]{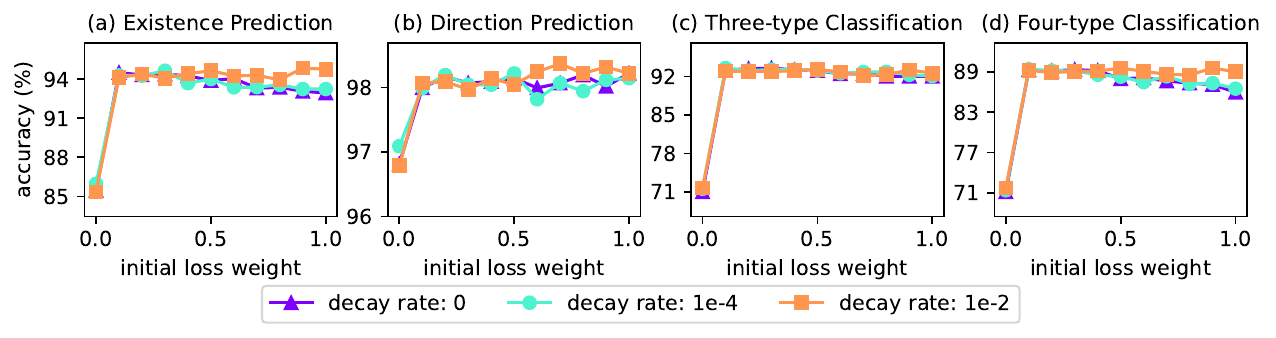}
\end{center}
\vspace{-3ex}
\caption{Link prediction accuracy (\%) under different initial loss weights and decay rates.}
\label{fig:sensi_edge}
\end{figure*}

\begin{figure*}[h]
\vspace{-0ex}
\begin{center}
\includegraphics[width=0.85
\textwidth]{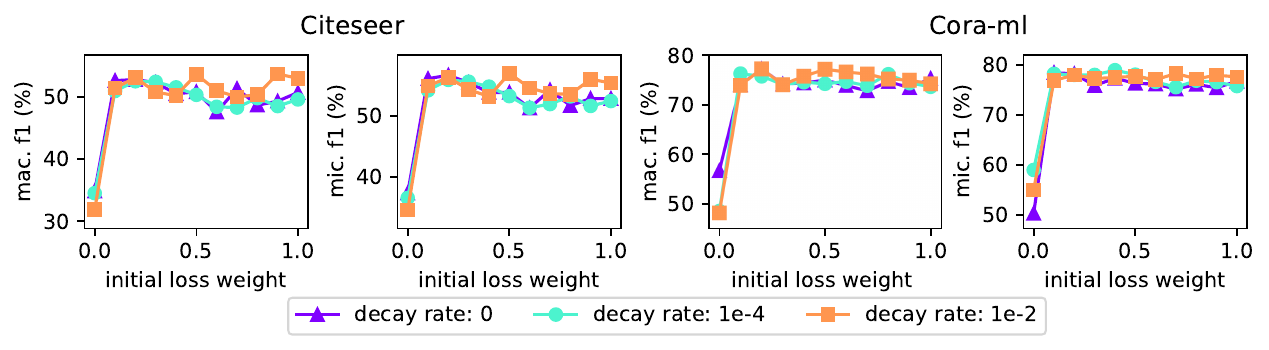}
\end{center}
\vspace{-3ex}
\caption{Self-supervised node classification macro $F_1$ (\%) and micro $F_1$ (\%) under different initial loss weights and decay rates.}
\label{fig:sensi_node}
\end{figure*}

\subsection{Distance Metrics}
\label{app:distance}

In Eq.~\eqref{eq:mul-decoder}, we employ the L1 distance as a measure to quantify the difference between the reconstructed matrix element and the corresponding element in the ground truth Hermitian adjacency matrix. However, it is worth noting that DUPLEX is robust to different distance metrics. In this section, we specifically utilize the L2 distance as an alternative distance metric and conduct link prediction experiments to compare the performance of DUPLEX using L2 distance. Table~\ref{table:abla-distance} illustrates the link prediction performance of DUPLEX utilizing L1 distance and L2 distance. From our observations, we find that DUPLEX with L1 distance generally performs slightly better than DUPLEX with L2 distance in most of the experiments conducted. Nevertheless, DUPLEX with L2 distance consistently demonstrates superior performance compared to other methods across all the experiments conducted.

\begin{table*}[h]
\caption{Link prediction ACC (\%) for four subtasks. Note that `EP' represents the `Existence Prediction' subtask, while `DP' for direction Prediction, `TP' for three-type classification and `FP' for four-type classification.}
\label{table:abla-distance}
\tabcolsep = 0.1cm
\begin{center}
\resizebox{0.6\linewidth}{!}{
\begin{tabular}{lcccccccc}
\toprule
\multirow{2}{*}{\textbf{\begin{tabular}[c]{@{}l@{}}Distance \\ Metric\end{tabular}}} & \multicolumn{4}{c}{\textbf{Citeseer}} & \multicolumn{4}{c}{\textbf{Cora}} \\
 & EP & DP & TP & FP & EP & DP & TP & FP \\ \midrule
DUPLEX(L1) & \textbf{95.7(0.5)} & \textbf{98.7(0.4)} & 94.8(0.2) & \textbf{91.1(1.0)} & \textbf{93.2(0.1)} & \textbf{95.9(0.1)} & \textbf{92.2(0.1)} & \textbf{88.4(0.4)} \\
DUPLEX(L2) & 95.0(0.6) & 98.6(0.5) & \textbf{95.4(0.6)} & 89.5(1.7) & 92.6(0.1) & 95.2(0.2) & 92.1(0.1) & 84.6(0.2) \\ \bottomrule
\toprule
\multirow{2}{*}{\textbf{\begin{tabular}[c]{@{}l@{}}Distance \\ metric\end{tabular}}} & \multicolumn{4}{c}{\textbf{Epinions}} & \multicolumn{4}{c}{\textbf{Twitter}} \\ 
 & EP & DP & TP & FP & EP & DP & TP & FP \\ \midrule
DUPLEX(L1) & 85.5(0.0) & \textbf{92.6(0.1)} & \textbf{88.9(0.0)} & 76.4(0.2) & 98.7(0.1) & \textbf{99.8(0.0)} & \textbf{99.2(0.1)} & 98.1(0.2) \\
DUPLEX(L2) & \textbf{86.5(0.1)} & 92.1(0.0) & 88.8(0.1) & \textbf{76.7(0.1)} & \textbf{98.9(0.1)} & 99.6(0.0) & 99.2(0.2) & \textbf{98.8(0.2)} \\ \bottomrule
\end{tabular}}
\end{center}
\vspace{-2ex}
\end{table*}

\end{document}